\documentclass[12pt]{article} 
\usepackage{amsthm}
\usepackage{geometry}
\usepackage{amsmath}
\usepackage{amsfonts}
\usepackage{graphicx}
\usepackage{rotating}
\usepackage{comment}
\usepackage{color}
\usepackage{hyperref}
\usepackage{mathrsfs}
\usepackage{fullpage}
\usepackage{bbm}
\usepackage{verbatim}
\usepackage{graphicx}
\usepackage{amssymb}
\usepackage{algorithm}
\usepackage{algpseudocode}
\usepackage[round,colon,authoryear]{natbib}

\newtheorem{theorem}{Theorem}[section]
\newtheorem{lemma}[theorem]{Lemma}

\def\RR{\mathbb{R}}

\def\calR{\mathcal{R}}
\def\calM{\mathcal{M}}
\def\calS{\mathcal{S}}

\DeclareMathOperator*{\argmin}{arg\,min}

\begin{document}
\title{Non-Convex Projected Gradient Descent for Generalized Low-Rank Tensor Regression}

\author{Han Chen$^\ast$, Garvesh Raskutti$^\dag$ and Ming Yuan$^\ast$\\
University of Wisconsin-Madison}

\date{}

\footnotetext[1]{
Morgridge Institute for Research and Department of Statistics, University of Wisconsin-Madison, 1300 University Avenue, Madison, WI 53706. The research of Han Chen and Ming Yuan was supported in part by NSF FRG Grant DMS-1265202, and NIH Grant 1-U54AI117924-01.}
\footnotetext[2]{
Departments of Statistics and Computer Science, and Optimization Group at Wisconsin Institute for Discovery, University of Wisconsin-Madison, 1300 University Avenue, Madison, WI 53706. The research of Garvesh Raskutti is supported in part by NSF Grant DMS-1407028.}
\maketitle

\newpage
\begin{abstract}
In this paper, we consider the problem of learning high-dimensional tensor regression problems with low-rank structure. One of the core challenges associated with learning high-dimensional models is computation since the underlying optimization problems are often non-convex. While convex relaxations could lead to polynomial-time algorithms they are often slow in practice. On the other hand, limited theoretical guarantees exist for non-convex methods. In this paper we provide a general framework that provides theoretical guarantees for learning high-dimensional tensor regression models under different low-rank structural assumptions using the projected gradient descent algorithm applied to a potentially non-convex constraint set $\Theta$ in terms of its \emph{localized Gaussian width}.  We juxtapose our theoretical results for non-convex projected gradient descent algorithms with previous results on regularized convex approaches. The two main differences between the convex and non-convex approach are: (i) from a computational perspective whether the non-convex projection operator is computable and whether the projection has desirable contraction properties and (ii) from a statistical upper bound perspective, the non-convex approach has a superior rate for a number of examples. We provide three concrete examples of low-dimensional structure which address these issues and explain the pros and cons for the non-convex and convex approaches. We supplement our theoretical results with simulations which show that, under several common settings of generalized low rank tensor regression, the projected gradient descent approach is superior both in terms of statistical error and run-time provided the step-sizes of the projected descent algorithm are suitably chosen.
\end{abstract}

\newpage
\section{Introduction}

Parameter estimation in high-dimensional regression has received substantial interest over the past couple of decades. See, e.g., \cite{BuhlmannVDGBook, HastieTibshiraniWainwrightBook}. One of the more recent advances in this field is the study of problems where the parameters and/or data take the form of a multi-way array or \emph{tensor}. Such problems arise in many practical settings \citep[see, e.g.,][]{CohenCollins12,LiLi10,Semerci14,SidNion10} and present a number of additional challenges that do not arise in the vector or matrix setting. In particular, one of the challenges associated with high-dimensional tensor regression models is how to define low-dimensional structure since the notion of rank is ambiguous for tensors \citep[see, e.g.,][]{KoldarBader}. Different approaches on how to impose low-rank and sparsity structure that lead to implementable algorithms have been considered. See, e.g., \cite{GandRecht11, MuGoldfarb14, RaskuttiYuan15, Tomioka13, YuanZhang14}, and references therein. All of the previously mentioned approaches have relied on penalized convex relaxation schemes and in particular, many of these different approaches have been encompassed by \cite{RaskuttiYuan15}. The current work complements these earlier developments by studying the non-convex projected gradient descent (PGD) approaches to generalized low-rank tensor regression.

While convex approaches are popular since greater theoretical guarantees have been provided for them, non-convex approaches have gained popularity as recently more theoretical guarantees have been provided for specific high-dimensional settings. See, e.g., \cite{Fan01, JainEtAl14, JainEtAl16, LohWainwright15}. Furthermore, even though non-convex problems do not in general lead to polynomial-time computable methods, they often work well in practice. In particular, inspired by the recent work of \cite{JainEtAl14, JainEtAl16} who demonstrated the effectiveness of non-convex projected gradient descent approaches for high-dimensional linear regression and matrix regression, we consider applying similar techniques to high-dimensional low-rank tensor regression problems with a generalized linear model loss function. 

Low-rankness in higher order tensors may occur in a variety of ways. To accommodate them, we develop a general framework which provides theoretical guarantees for projected gradient descent algorithms applied to tensors residing in general low-dimensional subspaces. Our framework relies on two properties ubiquitous in low-rank tensor regression problems: that the parameter space is a member of a class of subspaces super-additive when indexed over a partially ordered set; and that there exists a(n) (approximate) projection onto each subspace satisfying a certain contractive property. Assuming that the coefficient tensor lies in a low-dimensional subspace $\Theta$ satisfying these properties, we establish general risk bounds for non-convex projected gradient descent based methods applied to a generalized tensor regression model. Our main theoretical result shows that the Frobenius norm scales as $n^{-1/2}{w_G[\Theta \cap \mathbb{B}_{\rm F}(1)]}$, where $n$ is the sample size, $\mathbb{B}_{\rm F}(1)$ refers to the Frobenius-norm ball with radius $1$ and $w_G[\Theta \cap \mathbb{B}_{\rm F}(1)]$ refers to the \emph{localized Gaussian width} of $\Theta$. While statistical rates in terms of Gaussian widths are already established for convex regularization approaches \citep[see, e.g.,][]{ChandraRecht,RaskuttiYuan15}, this is to the best of our knowledge the first general upper bound for non-convex projected gradient descent in terms of a localized Gaussian width. 

A second  major contribution we make is to provide a comparison both in terms of statistical error rate and computation to existing convex approaches to low rank tensor regression. Using our statistical upper bound for non-convex projected gradient descent which is stated in terms of the localized Gaussian width of $\Theta$, we show explicitly that our upper bound for the non-convex approach is no larger (up to a constant) than the those for convex regularization schemes \citep[see, e.g., Theorem 1 of][]{RaskuttiYuan15}. To make this comparison more concrete, we focus on three particular examples of low-rank tensor structure: (i) low sum of ranks of each slice of a tensor; (ii) sparsity and low-rank structure for slices; and (iii) low Tucker rank. In case (i), both approaches are applicable and achieve the same rate of convergence. For case (ii), the non-convex approach is still applicable whereas a convex regularization approach is not naturally applicable. In case (iii) again both approaches are applicable but a superior statistical performance can be achieved via the non-convex method. We supplement our theoretical comparison with a simulation comparison. Our simulation results show that our non-convex projected gradient descent based approach compares favorably to the convex regularization approach using a generic \verb+cvx+ solver in terms of both run-time and statistical performance provided optimal step-size choices in the projected gradient descent and regularization parameters in the convex regularization approach are used. Furthermore the projected gradient descent scales to much larger-scale data than generic convex solvers.

To summarize, we make two major contributions in this work. Firstly, we provide a general upper bound for generalized tensor regression problems in terms of the localized Gaussian width of the constraint set $\Theta$. This upper bound is novel and this result can be applied for projected gradient descent with any constraint set $\Theta$ satisfying the two standard properties described above. Using this general result, our second major contribution is to provide a comparison to standard  convex regularization schemes. We show that unlike for vector and matrix problems where convex regularization schemes provably achieve the same statistical upper bounds as non-convex approaches, the more complex structure of tensors means that our non-convex approach could yield a superior statistical upper bound in some examples compared to previously considered convex regularization schemes. We also demonstrate through simulations the benefits of the non-convex approach compared to existing convex regularization schemes for various low-rank tensor regression problems.

The remainder of the paper is organized as follows: Section~\ref{SecBackground} introduces the basics of the low-rank tensor regression models we consider and introduces the projected gradient descent algorithm. Section~\ref{SecMain} presents the general theoretical results for non-convex projected gradient descent and specific examples are discussed in Section~\ref{SecExamples}. A simulation comparison between the convex and non-convex approach is provided in Section~\ref{SecSimulations} and proofs are provided in Section~\ref{SecProofs}.

\section{Methodology}

\label{SecBackground}

Consider a generalized tensor regression framework where the conditional distribution of a scalar response $Y$ given a covariate tensor $X\in \mathbb{R}^{d_1 \times d_2 \times \ldots \times d_N}$ is given by\begin{equation}
\label{EqnGLM}
p(Y| X, T) = h(Y)\exp\left\{Y\langle X, T \rangle - a(\langle X, T \rangle)\right\},
\end{equation}
%
%
%
where $a(\cdot)$ is a strictly convex log-partition function, $h(\cdot)$ is a nuisance parameter, and $T \in \mathbb{R}^{d_1 \times d_2 \times \ldots \times d_N}$ is the parameter tensor of interest. Typical examples of $a(\cdot)$ include $a(\theta) = \frac{1}{2}\theta^2$ leading to the usual normal linear regression, $a(\theta) = \log(1+ e^{\theta})$ corresponding to logistic regression, and $a(\theta) = e^{\theta}$ which can be identified with Poisson regression. The goal is to estimate $T$ based on the training data $\{(X^{(i)}, Y^{(i)}): 1\le i\le n\}$. For convenience we assume $(X^{(i)}, Y^{(i)})$'s are independent copies of $(X,Y)$. Hence the negative log-likelihood risk objective is:
\begin{equation}
\label{Eq:GLMLogLike}
\mathcal L (A) = \frac{1}{n}\sum_{i=1}^n \left[a(\langle X^{(i)}, A \rangle) - Y^{(i)} \langle X^{(i)}, A \rangle - \log h(Y^{(i)}) \right].
\end{equation}
 The notation $\langle\cdot,\cdot \rangle$ will refer throughout this paper to the standard inner product taken over appropriate Euclidean spaces. Hence, for $A\in \RR^{d_1\times\cdots\times d_N}$ and $B\in \RR^{d_1\times\cdots\times d_N}$:
$$\langle A , B \rangle=\sum_{j_1 = 1}^{d_1}\cdots\sum_{j_N = 1}^{d_N}A_{j_1,\ldots,j_N}B_{j_1,\ldots,j_N}\in \RR.$$

Using the standard notion of inner product, for a tensor $A$, $\|A\|_{\rm F} =\langle A, A\rangle^{1/2}$. And the empirical norm $\|\cdot\|_n$ for a tensor $A\in \RR^{d_1\times \cdots \times d_N}$ is define as:
$$
\|A\|_n^2 := {1\over n}\sum_{i=1}^n \langle A, X^{(i)}\rangle^2.
$$
Also, for any linear subspace $\mathcal{A} \subset \mathbb{R}^{d_1 \times d_2 \times...\times d_N}$, $A_{\mathcal{A}}$ denotes the projection of a tensor $A$ onto $\mathcal{A}$. More precisely 
\begin{equation*}
A_{\mathcal{A}} := \argmin_{M \in \mathcal{A}} \|A - M\|_{\rm F}.
\end{equation*}

\subsection{Background on tensor algebra}

One of the major challenges associated with low-rank tensors is that the notion of higher-order tensor decomposition and rank is ambiguous. See, e.g., \cite{KoldarBader} for a review. There are two standard decompositions we consider, the so-called canonical polyadic (CP) decomposition and the Tucker decomposition. The CP decomposition of a higher-order tensor is defined as the smallest number $r$ of rank-one tensors needed to represent a tensor $A\in \RR^{d_1\times d_2\times d_3 \times\ldots \times d_N}$:
\begin{equation}
\label{eq:cpdecomp}
A=\sum_{k=1}^r u_{k,1} \otimes u_{k,2} \otimes u_{k,3} \otimes \ldots \otimes u_{k,N}
\end{equation}
where $u_{k, m} \in \RR^{d_m}$, for $1 \leq k \leq r$ and $1 \leq m \leq N$.

A second popular decomposition is the so-called Tucker decomposition. The Tucker decomposition of a tensor $A\in \RR^{d_1\times d_2\times d_3 \times \ldots \times d_N}$ is of the form:
$$
A_{j_1j_2j_3\ldots j_N}=\sum_{k_1=1}^{r_1}\sum_{k_2=1}^{r_2}\sum_{k_3=1}^{r_3}\ldots\sum_{k_N=1}^{r_N} S_{k_1k_2k_3\ldots k_N}U_{j_1k_1,1}U_{j_2k_2,2}U_{j_3k_3,3}\ldots U_{j_Nk_N,N}
$$
so that $U_m \in \mathbb{R}^{d_m \times r_m}$ for $1 \leq m \leq N$ are orthogonal matrices and $S \in \mathbb{R}^{r_1 \times r_2 \times\ldots\times r_N}$ is the so-called core tensor where any two sub-tensors of $S$ satisfy basic orthogonality properties \citep[see, e.g.,][]{KoldarBader}. The vector $(r_1,r_2,r_3,\ldots,r_N)$ are referred to as the Tucker ranks of $A$. It is not hard to see that if (\ref{eq:cpdecomp}) holds, then the Tucker ranks $(r_1,r_2,r_3,\ldots,r_N)$ can be equivalently interpreted as the dimensionality of the linear spaces spanned by $\{u_{k,1}: 1\le k\le r\}$, $\{u_{k,2}: 1\le k\le r\}$, \ldots, and $\{u_{k,N}: 1\le k\le r\}$ respectively.

A convenient way to represent low Tucker ranks of a tensor is through \emph{matricization}. Denote by $\calM_1(\cdot)$ the mode-$1$ matricization of a tensor, that is $\calM_1(A)$ is the $d_1\times (d_2d_3\ldots d_N)$ matrix whose column vectors are the mode-$1$ fibers of $A\in \RR^{d_1\times d_2\times \ldots \times d_N}$. $\calM_2(\cdot)$,\ldots $\calM_N(\cdot)$ are defined in the same fashion. By defining
$$
{\rm rank}(\calM_m(A))=r_m(A),
$$
it follows that $(r_1(A), r_2(A),\ldots,r_N(A))$ represent the Tucker ranks of $A$. For later discussion, define $\calM_i^{-1}(\cdot)$ to be the inverse of mode-$i$ matricization, so
$$
\calM_i^{-1} : \mathbb{R}^{d_i \times (d_1\cdot d_2\cdots d_{i-1}\cdot d_{i+1}\cdots d_N)} \rightarrow \mathbb{R}^{d_1 \times d_2 \times...\times d_N},
$$
such that $\calM_i^{-1}(\calM_i(A)) = A$. 

Further, we define \emph{slices} of a tensor as follows. For an order-$3$ tensor $A$,
the $(1,2)$ slices of $A$ are the collection of $d_3$ matrices of $d_1\times d_2$
$$\left\{A_{\cdot \cdot j_3}:=(A_{j_1j_2j_3})_{1\le j_1\le d_1,1\le j_2\le d_2}: 1\le j_3\le d_3\right\}.$$

\subsection{Low-dimensional structural assumptions}

Since the notion of low-rank structure is ambiguous for tensors, we focus on three specific examples of low-rank structure. To fix ideas, we shall focus on the case when $N = 3$. Generalization to higher order cases is straightforward and omited for brevity. Firstly we place low-rank structure on the matrix slices. In particular first define:
$$
\Theta_1(r) = \left\{  A\in \RR^{d_1\times d_2 \times d_3}  :  \sum_{j_3=1}^{d_3} \text{rank}(A_{\cdot \cdot j_3} )\leq r \right\},
$$
which is the sum of the rank of the matrix slices.

Secondly we can impose a related notion where we take maximums of the rank of each slice and sparsity along the slices.
$$
\Theta_2(r, s ) = \left\{  A\in \RR^{d_1\times d_2 \times d_3}  :  \max_{j_3} \text{rank}(A_{\cdot \cdot j_3} )\leq r, \sum_{j_3=1}^{d_3} \mathbbm{1}(A_{\cdot \cdot  j_3} \neq 0)\leq s \right\}.
$$

Finally, we impose the assumption that all of the Tucker ranks are upper bounded:
$$
\Theta_3(r) = \left\{A\in \RR^{d_1\times d_2 \times d_3}  :  \max\{r_1(A), r_2(A), r_3(A) \}\leq r  \right\}.
$$
Note that all these low-dimensional structural assumption $\Theta_1(r)$, $\Theta_2(r, s )$ and $\Theta_3(r)$ are non-convex sets. In the next subsection we introduce a general projected gradient descent (PGD) algorithm for minimizing the generalized linear model objective \eqref{Eq:GLMLogLike} subject to the parameter tensor $A$ belonging to a potentially non-convex constraint set $\Theta$. 

\subsection{Projected Gradient Descent (PGD) iteration}

In this section we introduce the non-convex projected gradient descent (PGD) approaches developed in \cite{JainEtAl14,JainEtAl16}. The problem we are interested in is minimizing the generalized linear model objective \eqref{Eq:GLMLogLike} subject to $A$ belonging to a potentially non-convex set. The PGD algorithm for minimizing a general loss function $f(A)$ subject to the constraint $A \in \Theta$ is as follows:

\begin{algorithm}[H]
\caption{Projected Gradient Descent}
\label{CHalgorithm}
\begin{algorithmic}[1]
\State \textbf{Input : } data $\textbf{Y}, \textbf{X}$, parameter space $\Theta$, iterations $K$, step size $\eta$
\State \textbf{Initialize : } $k=0$, $\widehat {T}_{0} \in \Theta$
\For{ $k=1,2,\ldots, K$ }
\State $g_k = \widehat T_k - \eta \nabla f(\widehat T_k)$ \text{ (gradient step)}
\State $\widehat T_{k+1} = P_{\Theta}(g_k)$  or $\widehat T_{k+1} = \widehat P_{\Theta}(g_k)$ \text{ ((approximate) projection step)}
\EndFor
\State \textbf{Output : } $\widehat T_K$
\end{algorithmic}
\end{algorithm}

The notation $\widehat P_{\Theta}(\cdot)$ refers to an approximate projection on to $\Theta$ if an exact projection is not implementable. The PGD algorithm has been widely used for both convex and non-convex objectives and constraint sets. In our setting, we choose the negative log-likelihood for the generalized linear model as the function $f(A)$ to minimize while focusing on $\Theta = \Theta_1(r), \Theta_2(r,s)$ and $\Theta_3(r)$.
%

The projections we consider are all combinations of projections on to matrix or vector subspaces defined in \cite{JainEtAl14,JainEtAl16}. In particular, for a vector $v \in \mathbb{R}^d$, we define the projection operator $\tilde{P}_s(v)$ as the projection on to the set of $s$-sparse vectors by selecting the $s$ largest elements of $v$ in $\ell_2$-norm. That is:
\begin{equation*}
\tilde{P}_s(v) := \argmin_{\|z\|_{\ell_0} \leq s} \|z - v\|_{\ell_2}.
\end{equation*}
For a matrix $M \in \mathbb{R}^{d_1 \times d_2}$, let $\bar{P}_r(M)$ denote the rank-$r$ projection:
\begin{equation*}
\bar{P}_r(M) := \argmin_{{\rm rank}(Z) \leq r} \|Z - M\|_{\rm F}.
\end{equation*}
As mentioned in ~\cite{JainEtAl14,JainEtAl16}, this projection is also computable by taking the top $r$ singular vectors of $M$. For the remainder of this paper we use both of these projection operators for vectors and matrices respectively.

\section{Main Results}

\label{SecMain}

In this section we present our general theoretical results where we provide a statistical guarantee for the PGD algorithm applied to a low-dimensional space $\Theta$. 

\subsection{Properties for $\Theta$ and its projection}
To ensure the PGD algorithm converges for a given subspace $\Theta$, we view it as a member of a collection of subspaces $\{\Theta(t): t \in \Xi\}$ for some $\Xi\subset {\mathbb Z}_+^k$ and require some general properties of the collection. The index $t$ typically represents a sparsity and/or low-rank index and may be multi-dimensional. For example, $\Theta_1(r)$ is indexed by rank $r$ where
$$\Xi=\{0,\ldots,d_3\cdot\min\{d_1,d_2\}\}.$$
Similarly, $\Theta_{2}(r,s)$ is indexed by $t=(r,s)$ so that
$$\Xi=\{(0,0),\ldots, (\min\{d_1,d_2\},d_3)\},$$
and $\Theta_3(r)$ is indexed by rank $r$ so that
$$\Xi=\left\{0,\ldots, \max\left\{\min\{d_1,d_2d_3\},\min\{d_2,d_1d_3\},\min\{d_3,d_1d_2\}\right\}\right\}.$$

Note that the $\Xi$ is partially ordered where $a\ge (\le, <, >) b$ for two vectors $a$ and $b$ of conformable dimension means the inequality holds in an element-wise fashion.
\paragraph{Definition 1.} A set $\{\Theta(t): t \in \Xi\}$ is a {\it superadditive and partially ordered collection of symmetric cones} if
\begin{enumerate}
\item[(1)] each member $\Theta(t)$ is a {\it symmetric cone} in that if $z\in \Theta(t)$, then $c z \in \Theta(t)$ for any $c\in \RR$;
\item[(2)] the set is {\it partially ordered} in that for any $t_1 \leq t_2$, $\Theta(t_1)\subset \Theta(t_2)$;
\item[(3)] the set is {\it superadditive} in that $\Theta(t_1) + \Theta(t_2) \subset \Theta(t_1+t_2)$.
\end{enumerate}
\medskip
The first two properties basically state that we have a set of symmetric cones in the tensor space with a partial ordering indexed by $t$. The last property requires that the collection of subspaces be superadditive in that the Minkowski sum of any two subspaces is contained in the subspace of dimension that is the sum of the two lower dimensions.

Furthermore, we introduce the following property of contractive projection, for $P_{\Theta}$ or $\widehat P_{\Theta}$ in Algorithm 1, that is essential for the theoretical performance of the PGD algorithm. Again, we shall view these operators as members of a collection of operators $Q_{\Theta(t)}: \cup_{t} \Theta(t)\mapsto \Theta(t)$. The contractive projection property says that, when these operators are viewed as projections, projection onto a larger ``dimension'' incurs less approximation error {\it per dimension} compared to projection onto a smaller dimension, up to a constant factor. 

\paragraph{Definition 2.} We say that a set $\{\Theta(t): t \geq 0\}$ and corresponding operators $Q_{\Theta(t)}: \cup_{t} \Theta(t)\mapsto \Theta(t)$ satisfy the {\it contractive projection property} for some $\delta>0$, denoted by $\text{CPP}(\delta)$, if for any $t_1< t_2<t_0$, $Y\in \Theta(t_1)$, and $Z\in \Theta(t_0)$:
$$
\|Q_{\Theta(t_2)}(Z) - Z\|_{\rm F}\leq \delta\left \|  \frac{t_0-t_2}{t_0-t_1}  \right \|_{\ell_\infty}^{1/2} \cdot  \|Y - Z \|_{\rm F}.
$$
\medskip

It is clear that $\Theta_1(r)$ is isomorphic to rank-$r$ block diagonal matrices with diagonal blocks $A_{\cdot\cdot 1}$, $A_{\cdot\cdot 2}$,\ldots, $A_{\cdot\cdot d_3}$ so that $\{\Theta_1(r)\}$ satisfies Definition 1. It is also easy to verify that $\{\Theta_1(r)\}$ and its projections $\{P_{\Theta_1(r)}\}$ obey CPP$(1)$. Later, we will see in Lemmas \ref{ProjectionLemma1} and \ref{ProjectionLemma2} that these two properties are also satisfied by $\{\Theta_2(r,s)\}$ and $\{\Theta_3(r)\}$, and their appropriate (approximate) projections. 

\subsection{Restricted strong convexity}

Now we state some general requirements on the loss function, namely the restricted strong convexity and smoothness conditions (RSCS), that are another essential part for the guarantee of PGD performance \citep[see, e.g., ][]{JainEtAl14,JainEtAl16}.

\paragraph{Definition 3.} We say that a function $f$ satisfies {\it restricted strong convexity and smoothness conditions} $RSCS(\Theta,C_l,C_U)$ for a set $\Theta$, and $0<C_l<C_u<\infty$ if for any $A \in \Theta$,  $\nabla^2 f (A)$ is positive semidefinite such that for any $B \in \Theta$
$$
C_l \cdot  \|B\|_{\rm F}   \leq    \|\nabla^2 f (A) \cdot \text{vec}(B)\|_{\ell_2}  \leq  C_u  \cdot \|B\|_{\rm F},
$$
for some constants $C_l<C_u$, where $\nabla^2 f$ is the Hessian of $f$ on vectorized tensor.
\medskip

We first state the following Theorem about the PGD performance under general loss function which is a tensor version of the results in \cite{JainEtAl14,JainEtAl16}. Following similar steps to \cite{JainEtAl14,JainEtAl16}, we can derive the following result.

\begin{theorem}
\label{TheoremGeneralLoss}
Suppose that $\{\Theta(t): t \geq 0\}$ is a superadditive and partially ordered collection of symmetric cones, together with operators $\{P_{\Theta(t)}: t\ge 0\}$ which obey CPP$(\delta)$ for some constant $\delta>0$, and $f$ satisfies $RSCS(\Theta(t_0), C_l, C_u)$ for some constants $C_l$ and $C_u$. Let $\widehat T_K$ be the output from the $K$th iteration of applying PGD algorithm with step size $\eta=1/C_u$, and projection $P_{\Theta(t_1)}$ where
$$
t_1=\left\lceil{4 \delta^2 C_u^2 C_l^{-2}\over 1+ 4 \delta^2 C_u^2 C_l^{-2}}\cdot t_0\right\rceil.
$$
Then
$$\sup_{T\in \Theta(t_0-t_1)}\|\widehat T_K - T\|_{\rm F} \leq  {4 \eta C_u C_l^{-1} } \cdot   {\sup_{A\in \Theta(t_0)\cap \mathbb{B}_{\rm F}(1)}\left\langle \nabla f(T), A\right\rangle} + \epsilon,$$
for any
$$K \ge 2 C_u C_l^{-1} \log\left(\frac{ {\|T\|}_{\rm F}}{\epsilon}\right).$$
\end{theorem}

\subsection{Generalized linear models}
\label{SecGLM}

Now to use Theorem \ref{TheoremGeneralLoss} in a specific setting, we need to verify the conditions on $\{\Theta(t): t \geq 0\}$, $\{P_{\Theta(t)}: t\ge 0\}$ and $f$ satisfying $RSCS(\Theta,C_l,C_U)$, and choose the step-size in the PGD accordingly. 

First we turn our attention to the covariate tensor $(X^{(i)})_{i=1}^n$ where $X^{(i)} \in \mathbb{R}^{d_1 \times d_2 \times \ldots \times d_N}$ and how it relates to the $RSCS(\Theta,C_l,C_U)$. With slight abuse of notation, write
$${\rm vec}(X^{(i)}) \in \mathbb{R}^{d_1d_2\cdots d_N}$$
for $1 \leq i \leq n$ which is the vectorization of each tensor covariate $X^{(i)}$. For convenience let $D_N = d_1d_2\cdots d_N$. Further as mentioned for technical convenience we assume a Gaussian design of independent sample tensors $X^{(i)}$ s.t.
\begin{equation}
\label{GaussianCovariate}
{\rm vec}(X^{(i)}) \sim \mathcal{N}(0,\Sigma) \text{ where } \Sigma \in \mathbb{R}^{D_N \times D_N}.
\end{equation}
With more technical work our results may be extended beyond random Gaussian designs. We shall assume that $\Sigma$ has bounded eigenvalues. Let $\lambda_{\min}(\cdot)$ and $\lambda_{\max}(\cdot)$ represent the smallest and largest eigenvalues of a matrix, respectively. In what follows, we shall assume that
\begin{equation}
\label{AssCov}
c_{\ell}^2 \leq \lambda_{\min}(\Sigma) \leq \lambda_{\max}(\Sigma) \leq c_u^2,
\end{equation}
for some constants $0< c_\ell\le c_u<\infty$. For our analysis of the non-convex projected gradient descent algorithm, we define the condition number $\kappa = {c_u}/{c_l}$.

A quantity that emerges from our analysis is the \emph{Gaussian width} \citep[see, e.g.,][]{Gordon88} of a set $S \subset \mathbb{R}^{d_1 \times d_2 \times...\times d_N}$ which is defined to be:
\begin{equation*}
w_G(S) := \mathbb{E}\left(\sup_{A \in S} \langle A, G \rangle \right),
\end{equation*}
where $G \in \mathbb{R}^{d_1 \times d_2 \times \ldots \times d_N}$ is a tensor whose entries are independent $\mathcal{N}(0,1)$ random variables.
The Gaussian width is a standard notion of size or complexity of a subset of tensors $S$.

In addition to the positive semi-definiteness of the Hessian in the GLM model, the following Lemma extended a restricted upper and lower eigenvalue condition to the sample version of $\Sigma$ and hence guarantees the restricted strong convexity/ smoothness condition for GLM model with Gaussian covariates with quite general covariance structure. 

\begin{lemma}
\label{RestrictedEigenvalueLemma} 
Assume that \eqref{GaussianCovariate} and \eqref{AssCov} hold. For any $\tau >1$, there exist constants $c_1,c_2,c_3>0$ such that if $n \geq c_1 w_G^2[\Theta \cap \mathbb{B}_{\rm F}(1)]$, then with probability at least $1-c_2\exp(-c_3w_G^2[\Theta \cap \mathbb{B}_{\rm F}(1)])$, 
$$
\left(\tau^{-1} c_l\right)^2 \|A\|_{\rm F}^2 \leq   \frac{1}{n}  \sum_{i=1}^n \langle X^{(i)},A \rangle ^2 \leq (\tau c_u)^2 \|A\|_{\rm F}^2,\qquad \forall A\in\Theta.
$$
\end{lemma}
Lemma \ref{RestrictedEigenvalueLemma} is a direct consequence of Theorem 6 in \cite{BanerjeeChen15}. Using these definitions, we are in a position to state the main result for generalized linear models.


\begin{theorem}
\label{TheoremGLM}
Suppose that $\{\Theta(t): t \geq 0\}$ is a superadditive and partially ordered collection of symmetric cones, and together with operators $\{P_{\Theta(t)}: t\ge 0\}$ which obey CPP$(\delta)$ for some constant $\delta>0$. Assume that $\{(X^{(i)}, Y^{(i)}): i=1,\ldots, n\}$ follow the generalized linear model \eqref{EqnGLM} and $X^{(i)}$'s satisfy \eqref{GaussianCovariate} and \eqref{AssCov}, $\mathbb E |Y^{(i)}|^q\leq M_Y$ for some $q>2$ and $M_Y>0$, $1/\tau_0^2\leq \text{Var}(Y^{(i)})   \leq \tau_0^2$ for $i=1,\ldots, n$ and some $\tau_0>0$, and
$n > c_1w_G^2[\Theta(t_0) \cap \mathbb{B}_{\rm F}(1)]$ for some $t_0$ and $c_1>0$. Let $\widehat T_K$ be the output from the $K$th iteration of applying PGD algorithm to \eqref{Eq:GLMLogLike} with step size $\eta = (\tau c_u)^{-2}$ and projection $P_{\Theta(t_1)}$ where
$$t_1 = \left\lceil {4 \delta^2 \tau^8 \kappa^4\over 1+4 \delta^2 \tau^8 \kappa^4}\cdot t_0\right\rceil,$$
for any given $\tau>\tau_0$. Then there exist constants $c_2,c_3,c_4,c_5>0$ such that
$$\sup_{T\in \Theta(t_0-t_1)}\|\widehat T_K - T\|_{\rm F} \leq  \frac{c_5 \eta \tau^4\kappa^2c_u M_Y^{1/q}  } {\sqrt{n}}\cdot {w_G[\Theta(t_0) \cap \mathbb{B}_{\rm F}(1)] } + \epsilon,$$
with probability at least 
$$1 - Kc_2\exp\left\{-c_3 {w_G^2[\Theta(t_0) \cap \mathbb{B}_{\rm F}(1)] }\right\} - Kc_4 n^{-(q/2-1)}\log^q n,$$
for any
$$K \ge 2 \tau^4\kappa^2 \log\left(\frac{ {\|T\|}_{\rm F}}{\epsilon}\right).$$
\end{theorem}
\medskip

Notice that the statistical error we have is related to the Gaussian width of the intersection of a unit Frobenius ball and an (often non-convex) subset of low-dimensional structure $w_G[\Theta(t_0) \cap \mathbb{B}_{\rm F}(1)]$. The intersection of $\Theta(t_0)$ with $\mathbb{B}_{\rm F}(1)$ means we are \emph{localizing} the Gaussian width to a unit Frobenius norm ball around $T$. \emph{Localization} of the Gaussian width means a sharper statsitical upper bound can be proven and the benefits of localization in empirical risk minimization have been previously discussed in~\cite{BartlettBousquet05}. Later we will see how the main result leads to sample complexity bounds applied to $\Theta_2(r,s)$ and $\Theta_3(r)$. To the best of our knowledge this is the first general result that provides statistical guarantees in terms of the local Gaussian width of $\Theta(t_0)$ for the projected gradient descent algorithm. Expressing the upper bound in terms of the Gaussian width allows an easy comparison to already established upper bounds convex regularization schemes which we discuss in Section~\ref{SecGaussianGLM}.

The moment conditions on the response in Theorem \ref{TheoremGLM} are in place to ensure that the restricted strong convexity and restricted smoothness conditions are satisfied for a non-quadratic loss. When specialized under the normal linear regression, these conditions could be further removed. 

\subsection{Gaussian model and comparison to convex regularization}

\label{SecGaussianGLM}

Consider the Gaussian linear regression setting which corresponds to the GLM in Equation \eqref{EqnGLM} with $a(\theta) = \frac{\theta^2}{2}$. In particular
\begin{equation}
\label{EqnLinMod}
Y^{(i)} = \langle X^{(i)}, T \rangle + \zeta^{(i)},
\end{equation}
where $\zeta^{(i)}$'s are independent $\mathcal{N}(0, \sigma^2)$ random variables. Furthermore, substituting $a(\theta) = \frac{\theta^2}{2}$ into the GLM objective~\eqref{Eq:GLMLogLike}, we have the least-squares objective:
\begin{equation}
\label{EqnLeastSquares}
f(A) = \frac{1}{2n}\sum_{i=1}^n {(Y^{(i)} - \langle X^{(i)}, A \rangle)^2}.
\end{equation}
Now we state our main result for the normal linear regression.
\begin{theorem}
\label{Theorem0}
Suppose that $\{\Theta(t): t \geq 0\}$ is a superadditive and partially ordered collection of symmetric cones, and together with operators $\{P_{\Theta(t)}: t\ge 0\}$ which obey CPP$(\delta)$ for some constant $\delta>0$. Assume that $\{(X^{(i)}, Y^{(i)}): i=1,\ldots, n\}$ follow the Gaussian linear model \eqref{EqnLinMod} where $n > c_1w_G^2[\Theta(t_0) \cap \mathbb{B}_{\rm F}(1)]$ for some $t_0$ and $c_1>0$. Let $\widehat T_K$ be the output from the $K$th iteration of applying PGD algorithm to \eqref{Eq:GLMLogLike} with step size $\eta = (\tau c_u)^{-2}$ and projection $P_{\Theta(t_1)}$ where
$$t_1 = \left\lceil {4 \delta^2 \tau^8 \kappa^4\over 1+4 \delta^2 \tau^8 \kappa^4}\cdot t_0\right\rceil,$$
for any given $\tau>1$. Then there exist constants $c_2,c_3>0$ such that
$$\sup_{T\in \Theta(t_0-t_1)}\|\widehat T_K - T\|_{\rm F} \leq  \frac{8 \eta \tau^4\kappa^2 c_u \sigma}{\sqrt{n}} {w_G[\Theta(t_0) \cap \mathbb{B}_{\rm F}(1)] } + \epsilon,$$
with probability at least
$$1 -Kc_2\exp\left\{-c_3 {w_G^2[\Theta(t_0) \cap \mathbb{B}_{\rm F}(1)] }\right\},$$
for any
$$K\ge 2 \tau^4\kappa^2 \log\left(\frac{ {\|T\|}_{\rm F}}{\epsilon}\right).$$
\end{theorem}
\medskip
%
%

One of the focusses of this paper outlined in the introduction is to compare the non-convex PGD approach in tensor regression to the existing convex regularization approach analyzed in~\cite{RaskuttiYuan15} applied to the Gaussian linear model~\eqref{EqnLinMod}. In this section we first summarize the general result from~\cite{RaskuttiYuan15} and then provide a comparison to the upper bound for the non-convex PGD approach. In particular, the following estimator for $T$ is considered:
\begin{equation}
\label{EqnGeneral}
\widehat{T} \in \argmin_{A\in \RR^{d_1\times \cdots\times d_N}}\left\{\frac{1}{2n} \sum_{i=1}^{n} \|Y^{(i)} - \langle A, X^{(i)} \rangle \|_{\rm F}^2 + \lambda \mathcal{R}(A)\right\},
\end{equation}
where the convex regularizer $\mathcal{R}(\cdot)$ is a norm on $\RR^{d_1\times \cdots\times d_N}$, and $\lambda>0$ is a tuning parameter. The \emph{convex conjugate} for $\mathcal{R}$ (see e.g. \cite{Rockafellar} for details) is given by:
\begin{equation*}
\mathcal{R}^*(B) := \sup_{A\in \mathbb{B}_{\mathcal{R}}(1)} \langle A, B \rangle.
\end{equation*}
For example if $\mathcal{R}(A) = \|A\|_{*}$, then $\mathcal{R}^*(B) = \|B\|_{s}$. Following \cite{Neg10}, for a subspace $\Theta$ of $\RR^{d_1\times \cdots\times d_N}$, define its compatibility constant $s(\Theta)$ as
\begin{equation*}
s(\Theta) := \sup_{A \in \Theta/\{0\}} \frac{\mathcal{R}^2(A)}{\|A\|_{\rm F}^2},
\end{equation*}
which can be interpreted as a notion of low-dimensionality of $\Theta$.

~\cite{RaskuttiYuan15} show that if $\widehat{T}$ is defined by~\eqref{EqnGeneral} and the regularizer $\calR(\cdot)$ is \emph{decomposable} with respect to $\Theta$, then if
\begin{equation}
\lambda \geq 2 w_G(\mathbb{B}_{\calR}(1)),
\end{equation}
where recall that $ w_G(\mathbb{B}_{\calR}(1)) = \mathbb{E}\big(\sup_{A \in \mathbb{B}_{\calR}(1)} \langle A, G\rangle\big)$. Then according to Theorem 1 in~\cite{RaskuttiYuan15}, 
\begin{equation}
\label{eq:risk}
\max\left\{\|\widehat{T}-T\|_n, \|\widehat{T}-T\|_{\rm F}\right\} \lesssim \frac{\sqrt{s(\Theta)} \lambda}{\sqrt{n}}.
\end{equation}
with probability at least $1 - \exp(-c n )$ for some constant $c > 0$. In particular setting $\lambda = 2 w_G(\mathbb{B}_{\calR}(1))$,
\begin{equation*}
\max\left\{\|\widehat{T}-T\|_n, \|\widehat{T}-T\|_{\rm F}\right\} \lesssim \frac{\sqrt{s(\Theta)} w_G(\mathbb{B}_{\calR}(1))}{\sqrt{n}}.
\end{equation*}
The upper bound boils down to bounding two quantities, $s(\Theta)$ and $w_G(\mathbb{B}_{\calR}(1))$, noting that for comparison pursposes the subpace $\Theta$ in the convex case refers to $\Theta(t_0)$ in the non-convex case. In the next section we provide a qualitative comparison between the upper bound for the non-convex PGD approach and the convex regularization approach. To be clear, ~\cite{RaskuttiYuan15} consider multi-response models where the response $Y^{(i)}$ can be a tensor which are not considered in this paper. 

The upper bound for the convex regularization scheme scales as ${\sqrt{s(\Theta(t_0))} w_G[\mathbb{B}_{\calR}(1)]}/{\sqrt{n}}$ while we recall that the upper bound we prove in this paper for the non-convex PGD approach scales as $w_G[\Theta(t_0) \cap \mathbb{B}_{\rm F}(1)]/{\sqrt{n}}$. Hence how the Frobenius error for the non-convex and convex approach scales depends on which of the quantities ${\sqrt{s(\Theta(t_0))} w_G[\mathbb{B}_{\calR}(1)]}/{\sqrt{n}}$ or $w_G[\Theta(t_0) \cap \mathbb{B}_{\rm F}(1)]/{\sqrt{n}}$ is larger. It follows easily that $w_G[\Theta(t_0) \cap \mathbb{B}_{\rm F}(1)] \leq \sqrt{s(\Theta(t_0))} w_G[\mathbb{B}_{\calR}(1)]$ since
\begin{eqnarray*}
w_G[\Theta(t_0) \cap \mathbb{B}_{\rm F}(1)]  & = & \mathbb{E}\big[\sup_{A \in \Theta(t_0), \|A\|_{\rm F} \leq 1} \langle A, G\rangle \big]\\
& \leq & \mathbb{E}\big[\sup_{\calR(A) \leq \sqrt{s(\Theta(t_0))}} \langle A, G\rangle \big] \\
& = & \sqrt{s(\Theta(t_0))}\mathbb{E}\big[\sup_{\calR(A) \leq 1} \langle A, G\rangle \big] = \sqrt{s(\Theta(t_0))} w_G[\mathbb{B}_{\calR}(1)].
\end{eqnarray*}
The first inequality follows from the subspace compatibility constant since for all $A \in \Theta(t_0) \cap  \mathbb{B}_{\rm F}(1)$, $\calR(A) \leq \sqrt{s(\Theta(t_0))}\|A\|_{\rm F} \leq \sqrt{s(\Theta(t_0))}$ and the final equality follows since $\calR(\cdot)$ is a convex function. Therefore the non-convex upper bound is always no larger than the convex upper bound and the important question is whether there is a gap between the convex and non-convex bounds which implies a superior bound in the non-convex case. For examples involving sparse vectors and low-rank matrices as studied in e.g., \cite{BuhlmannVDGBook,JainEtAl14,JainEtAl16}, these two quantities end up being identical up to a constant. On the other hand for tensors, as we see in this paper for $\Theta_3(r)$, the Gaussian width using the non-convex approach is smaller which presents an additional benefit for the non-convex projection approach.

In terms of implementation, the regularizer $\mathcal{R}(\cdot)$ needs to be defined in the convex approach and the important question is whether the convex regularizer is implementable for the low-dimensional structure of interest. For the non-convex approach, the important implementation issue is whether exact or approximate projection that satisfies the contractive projection property is implementable. These implementation issues have been resolved in the vector and matrix cases \citep[see, e.g.,][]{JainEtAl14,JainEtAl16}. In this paper, we focus on whether they apply in the low-rank tensor case under the low-dimensional structure $\Theta_1$, $\Theta_2$ and $\Theta_3$.

\section{Specific low tensor rank structure}

\label{SecExamples}

In this section, we apply Theorem~\ref{Theorem0} (and by extension Theorem~\ref{TheoremGLM}) to $\Theta_1(r)$, $\Theta_2(r,s)$ and $\Theta_3(r)$ and compare our theoretical result to the result achieved by the convex regularization approach. Recall that $\Theta_1(r)$ is isomorphic to rank-$r$ block diagonal matrices with diagonal blocks $A_{\cdot\cdot 1}$, $A_{\cdot\cdot 2}$,\ldots, $A_{\cdot\cdot d_3}$ so that its treatment is identical to the case of low rank matrix estimation. See \cite{JainEtAl16} for further discussions. We shall focus on $\Theta_2(r,s)$ and $\Theta_3(r)$ instead. To prove upper bounds using Theorem \ref{Theorem0} we find an exact or approximate projection $P_{\Theta(t)}$, prove the contractive projection property and then find an upper bound on the Gaussian width $w_G[\Theta(t) \cap \mathbb{B}_{\rm F}(1)]$.

\subsection{Low-rank structure for matrix slices}


Recall that
$$
\Theta_2(r, s ) = \left\{  A\in \RR^{d_1\times d_2 \times d_3}  :  \max_{j_3} \text{rank}(A_{\cdot \cdot j_3} )\leq r, \sum_{j_3=1}^{d_3} \mathbbm{1}(A_{\cdot \cdot j_3} \neq 0)\leq s \right\}.
$$
We define $P_{\Theta_2(r, s )}$ as a two-step projection:
\begin{enumerate}
\item[(1)] for each matrix slice $A_{\cdot \cdot j_3}$ where $1 \leq j_3 \leq d_3$, let $\tilde{A}_{\cdot \cdot j_3}$ be the best rank $r$ approximation of $A_{\cdot \cdot j_3}$; 
\item[(2)] to impose the sparsity condition, retain $s$ out of $d_3$ slices with the largest magnitude $\|\tilde{A}_{\cdot \cdot j_3}\|_{\rm F}$, and zero out all other slices.

\end{enumerate}
As discussed earlier both steps are easily computable using thresholding and SVD operators as discussed in \cite{JainEtAl14,JainEtAl16}.The following lemma proves that the contractive property of projection onto $\Theta_2(r,s)$ holds for our $P_{\Theta_2(r, s)}$.

\begin{lemma}
\label{ProjectionLemma1}
Let the projection operator $P_{\Theta_2(r, s)}$ be defined above. Suppose $Z \in \Theta_2(r_0,s_0) $, and $r_1<r_2< r_0, s_1<s_2< s_0$. Then for any $Y\in \Theta_2(r_1,s_1)$,  we have 
$$
\|P_{ \Theta_2(r_2,s_2)}(Z) - Z\|_{\rm F}  \leq  (\alpha +\beta +\alpha\beta)\cdot \|Y - Z\|_{\rm F}.
$$
where $\alpha = \sqrt{(s_0-s_2)/(s_0-s_1)}$, $\beta = \sqrt{(r_0-r_2)/(r_0-r_1)}$.
\end{lemma}

Consequently we have the following Theorem:
\begin{theorem}
\label{Theorem1}
Let $\{X^{(i)}, Y^{(i)}\}_{i=1}^n$ follow a Gaussian linear model as defined by~\eqref{EqnLinMod} with $T \in \Theta_2(r,s)$ and
$$
n \geq c_1\cdot sr(d_1+d_2 +\log d_3)
$$ 
for some constant $c_1>0$. Then, applying the PGD algorithm with step size $\eta = (\tau c_u)^{-2}$ and projection $P_{\Theta(r',s')}$ where
$$s' = \lceil 36 \tau^8\kappa^4 s \rceil,\qquad {\rm and}\qquad r' = \lceil 36\tau^8 \kappa^4 r \rceil, $$
guarantees that, with probability at least $1- Kc_2 \exp\{-c_3 \max(d_1,d_2,\log d_3)\}$, after $K \ge 2\tau^4\kappa^2 \log({\|T\|}_{\rm F}/\epsilon)$ iterations,
$$\|\widehat T_K - T\|_{\rm F} \leq  c_4\sigma\sqrt{\frac{sr\max\{d_1,d_2,\log(d_3)\}}{n}} + \epsilon $$
for any $\tau>1$, and some constants $c_2,c_3,c_4>0$.
\end{theorem}

The convex regularization approach defined by~\cite{RaskuttiYuan15} is not directly applicable for $\Theta_2(r, s)$ since there is no suitable choice of regularizer that imposes both low-rankness of each slice and sparsity. Therefore we discuss the convex regularization approach applied to the parameter space $\Theta_1(r)$ for which a natural choice of regularizer is:
\begin{equation*}
\mathcal{R}_1(A) = \sum_{j_3 = 1}^{d_3} \| A_{\cdot \cdot j_3} \|_{\ast},
\end{equation*}
where $\|\cdot\|_{\ast}$ refers to the standard nuclear norm of a matrix. Let $\widehat{T}$ be an estimator corresponding to the minimizer of the regularized least-squares estimator defined by~\eqref{EqnGeneral} with regularizer $\mathcal{R}_1(A)$. Lemma 6 in ~\cite{RaskuttiYuan15} proves that
\begin{equation*}
\|\widehat{T} - T\|_{\rm F} \lesssim \sqrt{\frac{{r \max(d_1, d_2, \log d_3)}}{{n}}}.
\end{equation*}

Notice that both $\Theta_1(r)$ and $\Theta_2(r, s ) $ focus on the low-rankness of matrix slices of a tensor, and actually $\Theta_1(\cdot)$ can be seen as relaxation of  $\Theta_2(\cdot,\cdot )  $ since $\Theta_2(s,r) \subset \Theta_1(sr)$. Theorem \ref{Theorem1} guarantees that, under the restriction of sparse slices of low-rank matrices, PGD achieves the linear convergence rate with the statistical error of order
$$\sqrt{sr\max\{d_1,d_2,\log(d_3)\}\over n}.$$
If we compare this result with the risk bound of the convex regularization approach where the true tensor parameter lies in $\Theta_1(r)$ we see that replacing $r$ by $sr$ yields the same rate which makes some intuitive sense in light of the observation that $\Theta_2(s,r) \subset \Theta_1(sr)$. 

\subsection{Low Tucker ranks}

We now consider the general set of tensors with low Tucker rank:
$$
\Theta_3 (r) = \left\{A\in \RR^{d_1\times d_2 \times d_3}  :  \max\{r_1(A), r_2(A),r_3(A) \}\leq r  \right\}.
$$
Although we focus on $N=3$, note that $\Theta_3 (r)$ can be easily extended to general and $N$ and we also consider $N=4$ in the simulations.

To define the projection $P_{\Theta_3(r)}$ on to $\Theta_3(r)$, we exploit the connection between Tucker ranks and ranks of different matricizations mentioned earlier. Recall that the matricization operator $\mathcal{M}_j$ maps a tensor to a matrix and the inverse operator $\mathcal{M}_j^{-1}$ maps a matrix back to a tensor. Let $\bar{P}_r(M)$ be the low-rank projection operator that maps a matrix $M$ to its best rank $r$ approximation. Then we can define the approximate projection $\widehat P_{\Theta_3(r)}$ as follows:
\begin{eqnarray}
\label{TuckerProj}
\widehat P_{\Theta_3(r)}(A) := (\mathcal{M}_3^{-1}\circ \bar{P}_r\circ\mathcal{M}_3)\circ(\mathcal{M}_2^{-1}\circ\bar{P}_r\circ\mathcal{M}_2)\circ(\mathcal{M}_1^{-1}\circ\bar{P}_r\circ\mathcal{M}_1)(A).
\end{eqnarray}
The order of which matricization is performed is nonessential. Similar to before, we have the following projection lemma to be essential in the analysis of PGD applied to the restricted parameter space $\Theta_3$.

\begin{lemma}\label{ProjectionLemma2}
Suppose $Z \in \Theta_3(r)$, and $r_1<r_2 <r_0$. Then for any $Y\in \Theta_3(r_1)$,  we have
$$
\|\widehat P_{ \Theta_3(r_2)}(Z) - Z\|_{\rm F}  \leq (3\beta +3\beta^2 + \beta^3) \|Y - Z\|_{\rm F}
$$
where
$\beta = \sqrt{(r_0-r_2)/(r_0-r_1)}$.
\end{lemma}

This allows us to derive the following result for the PGD algorithm applied with projection operator $\widehat P_{ \Theta_3(r')}(\cdot)$.
\begin{theorem}
\label{Theorem2}
Let $\{X^{(i)}, Y^{(i)}\}_{i=1}^n$ follow a Gaussian linear model as defined by~\eqref{EqnLinMod} with $T \in \Theta_3(r)$ and
$$
n \geq c_1\cdot  r\cdot \min\{d_1 + d_2d_3, d_2 + d_1d_3, d_3+d_1d_2\},
$$ 
for some constant $c_1>0$. Then, applying the PGD algorithm with step size $\eta = (\tau c_u)^{-2}$ and projection $\widehat P_{\Theta_3(r')}$ where
$$r' =\lceil 196\tau^8 \kappa^4 r \rceil,$$
guarantees that, with probability at least $1- Kc_2 \exp\{-c_3 \min(d_1 + d_2d_3, d_2 + d_1d_3, d_3+d_1d_2)\}$, after $K \ge 2\tau^4\kappa^2 \log({\|T\|}_{\rm F}/\epsilon)$ iterations,
$$\|\widehat{T}_K - T\|_{\rm F} \leq  c_4\sigma \sqrt{\frac{r\cdot \min\{ d_1+d_2d_3, d_2 + d_1d_3, d_3+ d_1d_2\}}{n}} + \epsilon $$
for any $\tau>1$, and some constants $c_2,c_3,c_4>0$.
\end{theorem}
In \cite{RaskuttiYuan15}, the following convex low-rankness regularizer is considered for the space $\Theta_3(r)$:
$$
\calR_2(A)={1\over 3}\left(\|\calM_1(A)\|_\ast+\|\calM_2(A)\|_\ast+\|\calM_3(A)\|_\ast.\right).
$$
Let $\widehat{T}$ be an estimator corresponding to the minimizer of the regularized least-squares estimator defined by~\eqref{EqnGeneral} with regularizer $\mathcal{R}_2(A)$. Lemma 10 in ~\cite{RaskuttiYuan15} proves that
\begin{equation*}
\|\widehat{T} - T\|_{\rm F} \lesssim \sqrt{\frac{{r\cdot \max(d_1+ d_2d_3, d_2 + d_1d_3, d_3 + d_1d_2)}}{{n}}}.
\end{equation*}
This shows that convex relaxation in this particular case has greater mean-squared error since the minimum is replaced by the maximum. The underlying reason is that the non-convex PGD approach selects the optimal choice of matricization whereas the convex regularization approach takes an average of the three matricizations which is sub-optimal. For instance if $d_1 = d^2$ and $d_2 = d_3 = d$, the convex regularization scheme achieves a rate of $n^{-1/2}r^{1/2}d^{3/2}$ whereas the non-convex approach achieves the much sharper rate of $n^{-1/2}r^{1/2}d$.

\section{Simulations}

\label{SecSimulations}

In this section, we provide a simulation study that firstly verifies that the non-convex PGD algorithm performs well in solving least-squares, logistic and Poisson regression problems and then compares the non-convex PGD approach with the convex regularization approach we discussed earlier. Our simulation study includes both third and fourth order tensors. For the purpose of illustration, we consider the balanced-dimension situation where $d = d_1 = d_2 =d_3 (= d_4)$, and hence the number of elements is $p=d^3$ for a third order tensor and $p=d^4$ for a fourth order tensor. 

\subsection{Data generation}
We first describe three different ways of generating random tensor coefficient $T$ with different types of low tensor rank structure.
\begin{enumerate}
\item (Low CP rank)
Generate three independent groups of $r$ independent random vectors of unit length, $\{u_{k,1}\}_{k=1}^r$, $\{u_{k,2}\}_{k=1}^r$ and $\{u_{k,3}\}_{k=1}^r$. To do this we perform the SVD of a Gaussian random matrix three times and keep the $r$ leading singular vectors, and then compute the outer-product 
$$
T =\sum_{k=1}^r u_{k,1} \otimes u_{k,2} \otimes u_{k,3}.
$$
The $T$ produced in this way is guaranteed to have CP rank at most $r$. This can easily be extended to $N=4$.
\item (Low Tucker rank)
Generate $M_{d\times d \times d}$ with i.i.d. $\mathcal{N}(0,1)$ elements and then do approximate Tucker rank-$r$ projection (successive low rank approximation of mode-1, mode-2 and mode-3 matricization) to get $T = \widehat P_{\Theta_3(r)}(M)$. The $T$ produced in this way is guaranteed to have largest element of Tucker rank at most $r$. Once again this is easily extended to the $N=4$ case.
\item (Sparse slices of low-rank matrices)
In this case $N = 3$. Generate $s$ slices of random rank-$r$ matrices, (with eigenvalues all equal to one and random eigenvectors), and fill up the remaining $d-s$ slices with zero matrices to get $d\times d \times d$ tensor $T$. The $T$ produced in this way is guaranteed to fall in $\Theta_2(r,s)$.
\end{enumerate}

Then we generate covariates $\{X^{(i)}\}_{i=1}^n$ to be i.i.d random matrices filled with i.i.d $\mathcal{N}(0,1)$ entries. 
Finally, we simulate three GLM model, the Gaussian linear model, logistic regression and Poisson regression as follows.

\begin{enumerate}
\item (Gaussian linear model) 
We simulated noise $\{\epsilon^{(i)}\}_{i=1}^n$ independently from $ \mathcal{N}(0,\sigma^2)$ and we vary $\sigma^2$.
The noisy observation is then
$$
Y^{(i)} = \langle X^{(i)} , T \rangle + \epsilon^{(i)}.
$$
\item (Logistic regression)
We simulated Binomial random variables:
$$
Y^{(i)} \sim \text{Binomial} (m, p_i) , \text{ where } p_i = \text{logit}( \alpha \cdot \langle X^{(i)}, T\rangle)
$$
\item (Poisson regression)
We simulated
$$
Y^{(i)} \sim \text{Poisson} (\lambda_i) , \text{ where } \lambda_i = m \exp( \alpha \cdot \langle X^{(i)}, T\rangle )
$$
\end{enumerate}

\subsection{Convergence of PGD under restricted tensor regression}
 
\subsubsection{Third order tensors}
The first set of simulations investigates the convergence performance of PGD under various constraints and step sizes for three different types of low-rankness. One of the important challenges when using the projected gradient descent algorithm is choosing the step-sizes (just like selecting the regularization parameter for convex regularization schemes) and the step-size choices stated in Theorem~\ref{TheoremGLM} depend on non-computable parameters (e.g. $c_u, c_{\ell},...$). In the first two cases (see cases below), PGD with approximate projection $\widehat P_{\Theta_3(r')}$ were applied with different choices of $(r', \eta)$  while in the third case the PGD with exact projection $ P_{\Theta_2(r', s')}$ were adopted with different choices of $(r', s', \eta)$.

\begin{enumerate}
\item[Case 1a:] (Gaussian) Low CP Rank with $p=50^3$, $n =4000$, $r=5$, $\sigma = 0.5$ (SNR = 4.5);
\item[Case 2a:] (Gaussian) Low Tucker Rank with $p=50^3$, $n=4000$, $r=5$, $\sigma = 5$ (SNR = 7.2);
\item[Case 3a:] (Gaussian) Slices of Low-rank Matrices with $p=50^3$, $n=4000$, $r=5$, $s= 5$, $\sigma = 1$ (SNR = 5.2). 
\end{enumerate}

Figures \ref{figure1}, \ref{figure2} and \ref{figure3} plot normalized rooted mean squared error (rmse) ${\|\widehat T -T\|_{\rm F}/ \|T\|_{\rm F}}$ versus number of iterations, showing how fast rmse decreases as the number of iterations increases, under different  $(r', \eta)$  or $(r', s', \eta)$. Notice that here we only show the simulation result for one typical run for each case 1, 2 and 3 since the simulation results are quite stable.

Overall, the plots show the convergence of rmse's, and that the larger the $r'$ or $s'$ is, the greater the converged rmse will be, meaning that misspecification of rank/sparsity will do harm to the performance of PGD. In terms of the choice of step size, the experiments inform us that if $\eta$ is too large, the algorithm may not converge and the range of tolerable step-size choices varies in different cases. In general, the more misspecified the constraint parameter(s) is(are), the lower the tolerance for step size will be. On the other hand, as we can see in all cases, given $\eta$ under a certain tolerance level, the larger the $\eta$ is, the faster the convergence will be.

\subsubsection{Fourth order tensors}
Although we have focused on third order tensor for brevity, our method applies straightforwardly to higher order tensors. For illustration, we considered the following two examples which focus on estimating fourth order low rank tensors.

\begin{enumerate}
\item[Case 4a:] (Gaussian) Low CP Rank with $p=20^4$, $n =4000$, $r=5$, $\sigma = 0.5$  (SNR = 4.4); 
\item[Case 5a:] (Gaussian) Low Tucker Rank with $p=20^4$, $n=4000$, $r=5$, $\sigma = 5$  (SNR = 7.4).
\end{enumerate}

Figure \ref{figure4} plots rmse vs number of iterations for Case 4a and Case 5a using $\eta = 0.2$ under various choices of low-rankness constraint parameter $r'$. In general the convergence behavior for Case 4a and Case 5a are similar to those for Case 1a and Case 2a.

\begin{figure} 
\centering
\includegraphics[width=0.8\linewidth]{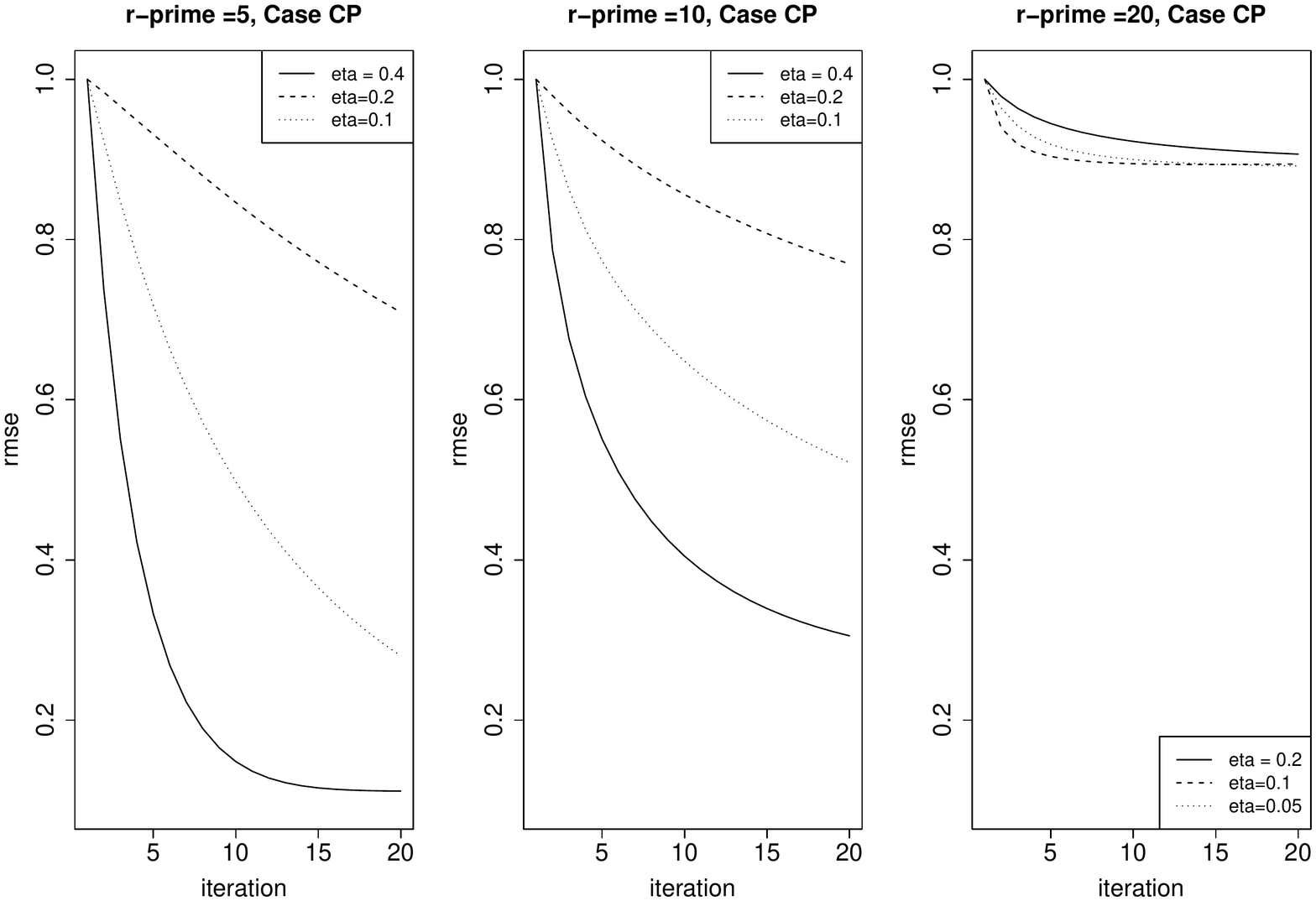}
\caption{Case 1a: Low CP rank}
\label{figure1}
\bigskip 
\includegraphics[width=0.8\linewidth]{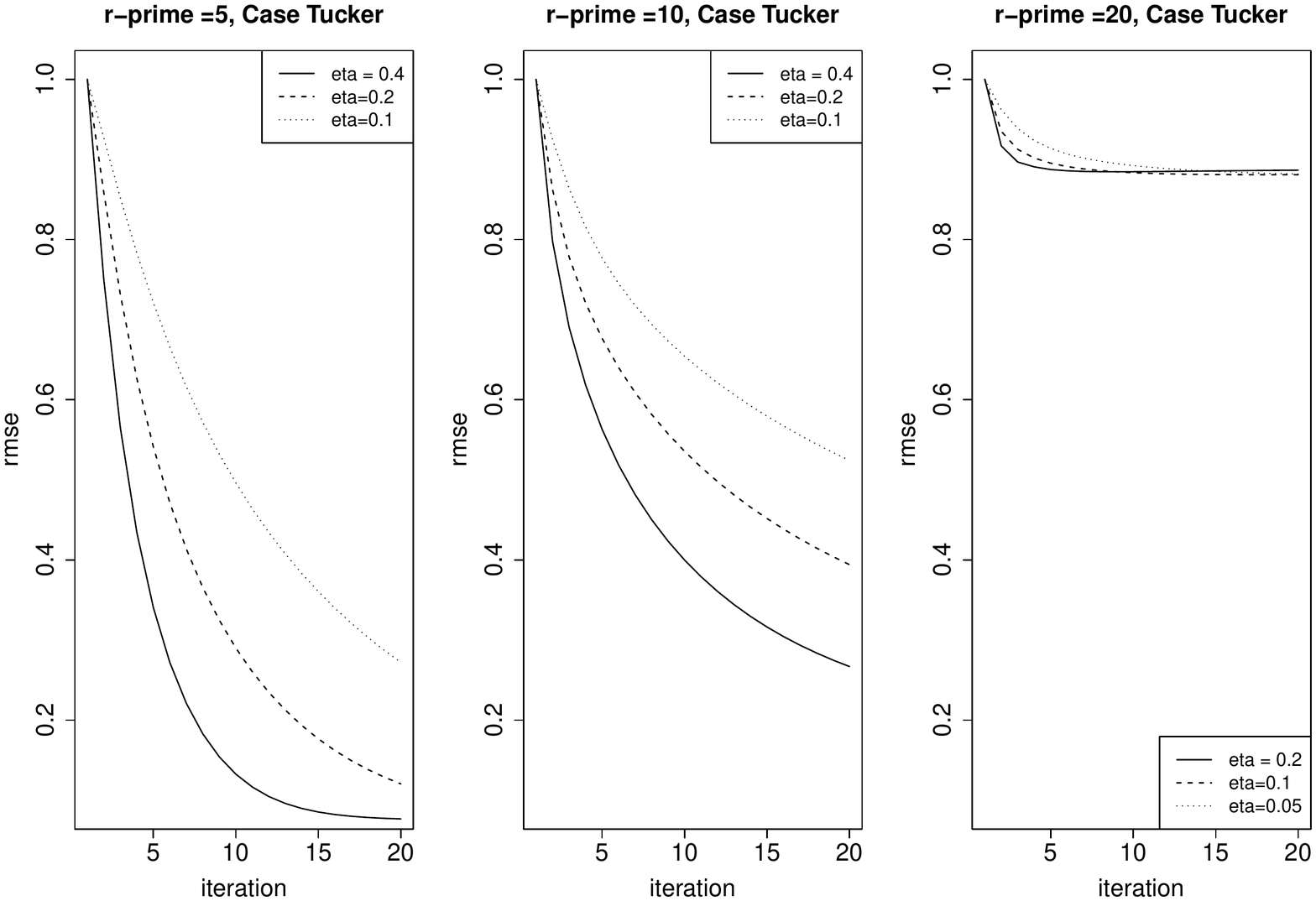}
\caption{Case 2a: Low Tucker rank}
\label{figure2}
\end{figure}

\begin{figure} 
\centering
\includegraphics[width=1.0\linewidth]{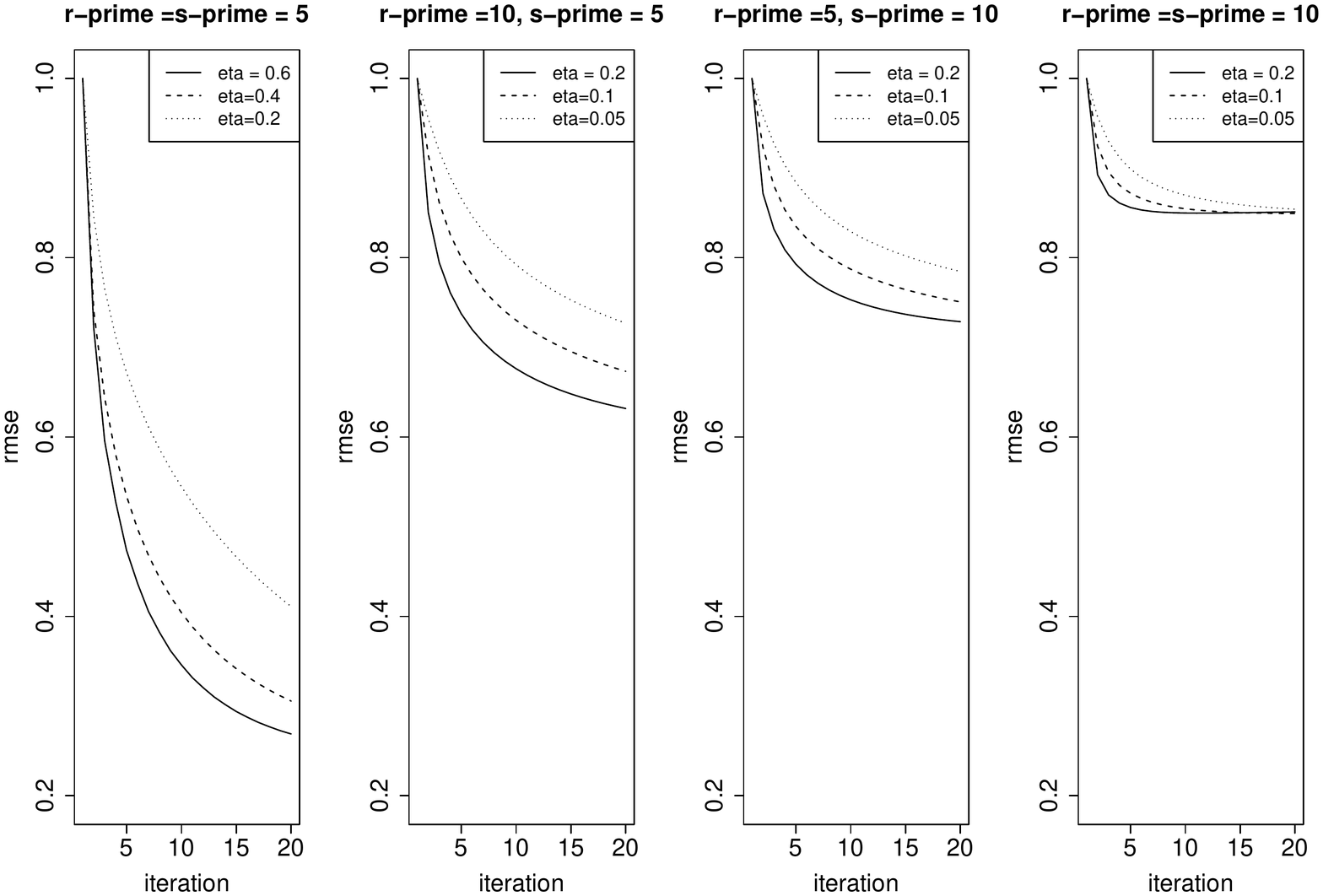}
\caption{Case 3a: Sparse slices of low-rank matrices}
\label{figure3}
\includegraphics[width=0.65\linewidth]{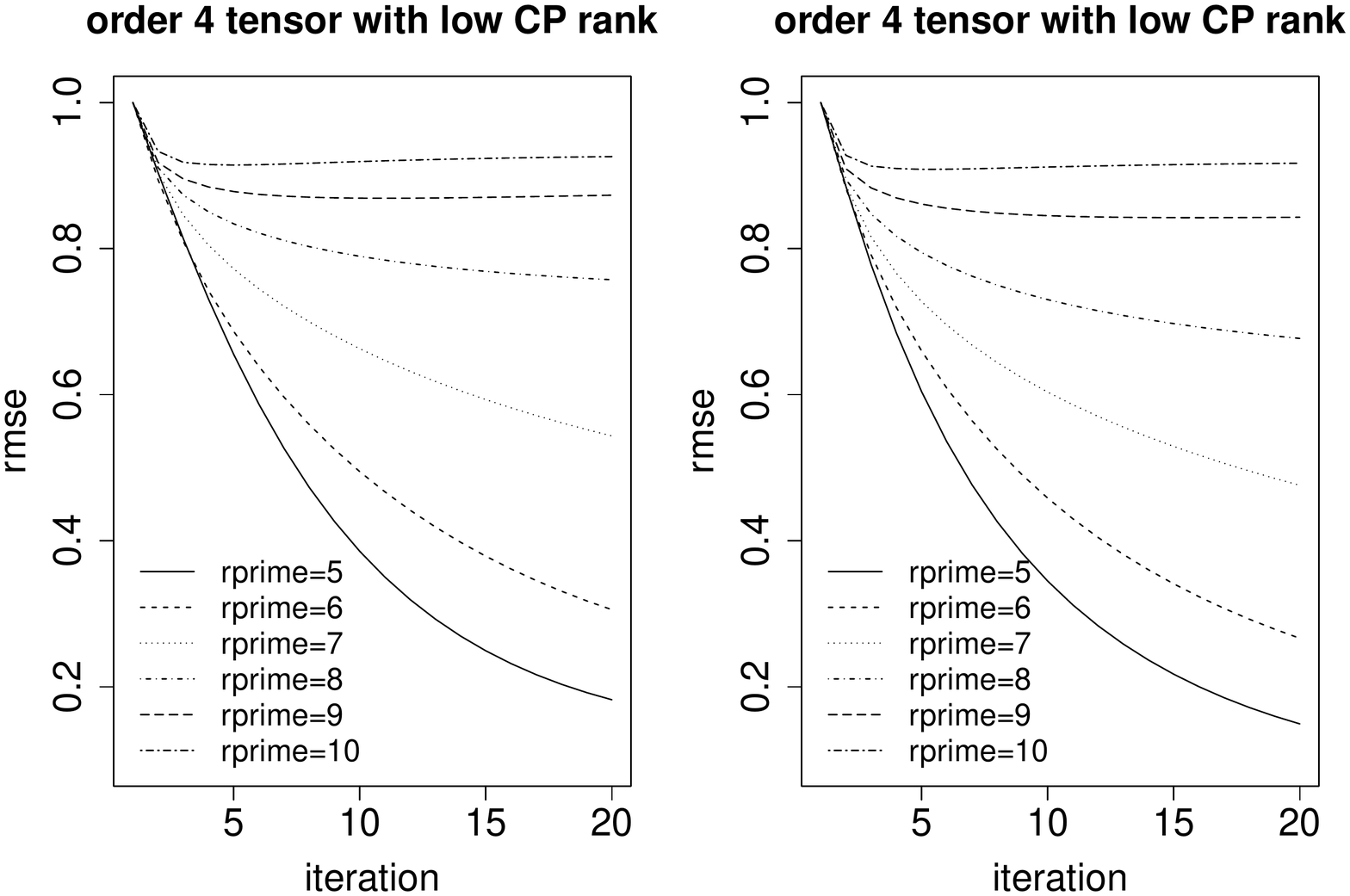}
\caption{Case 4a, 5a: 4th order tensor }
\label{figure4}
\end{figure}

\begin{figure} 
\centering
\includegraphics[width=0.8\linewidth]{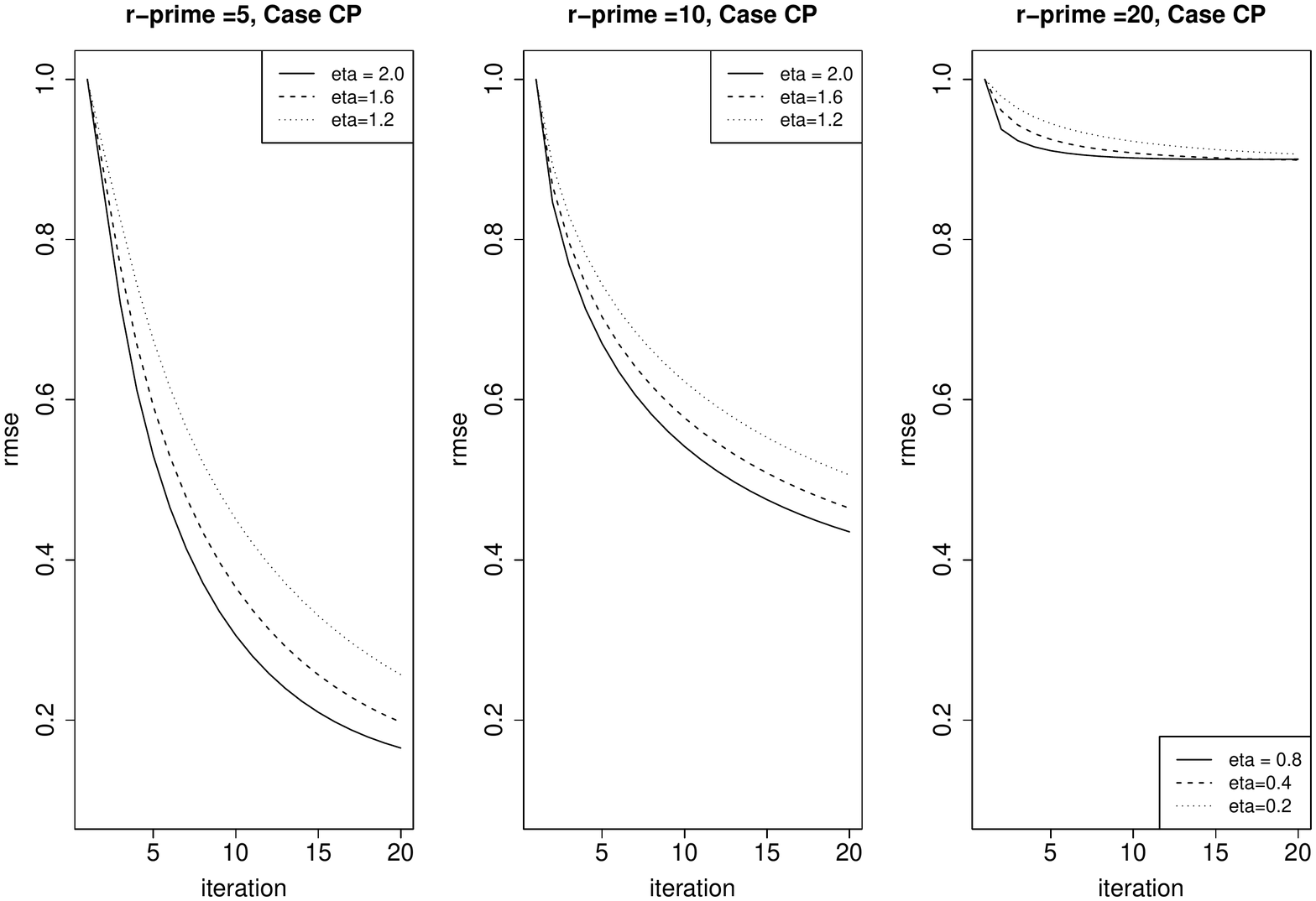}
\caption{Case 1b: (Logistic) Low CP rank}
\label{figure5}
\bigskip 
\includegraphics[width=0.8\linewidth]{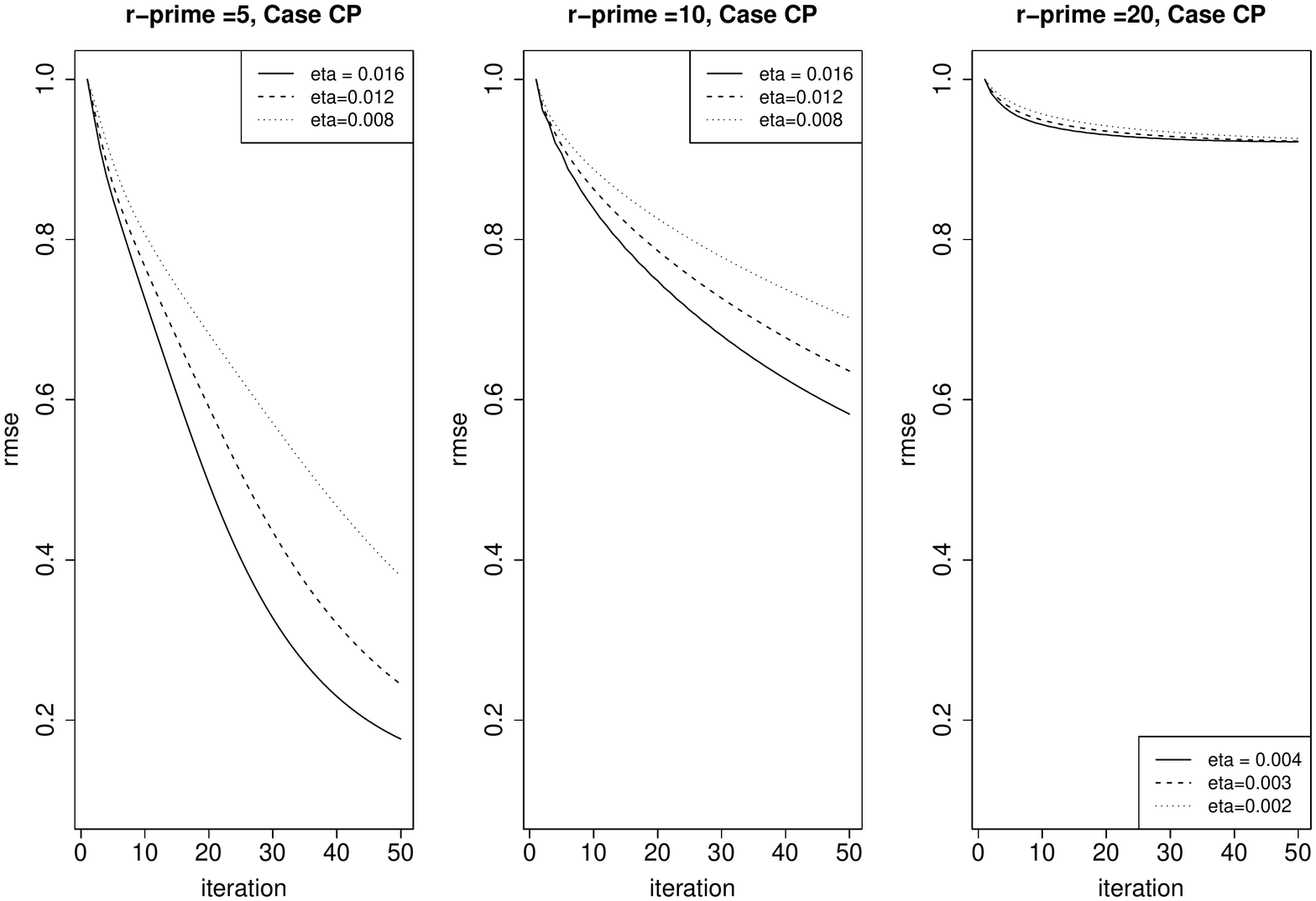}
\caption{Case 1c: (Poisson) Low CP rank }
\label{figure6}
\end{figure}

\subsubsection{Logistic and Poisson regression}
In the next set of simulations, we study the convergence behavior of the PGD applied to logistic and Poisson regression situation. 

\begin{enumerate}
\item[Case 1b:] (Logistic) Low CP Rank with $p=50^3$, $n =4000$, $r=5$, $m=22$, $\alpha = 1$ (SNR = 4.1);
\item[Case 1c:] (Poisson) Low CP Rank with $p=50^3$, $n =4000$, $r=5$, $m=10$, $\alpha = 0.5$ (SNR =  6.5).
\end{enumerate}

The results presented in Figures \ref{figure5} and \ref{figure6} exhibit similar pattern of convergence as in Figure \ref{figure1}. We observe also that in the case of low Tucker rank and sparse slices of low-rank matrices, logistic and Poisson regression have similar convergence behavior to least-squares regression. In general, a relaxed projection step is inferior to using the true rank parameter for projection. Once again as the step-size increases, the convergence will speed up until the step size becomes too large to guarantee convergence.

\subsection{Comparison of non-convex PGD to convex regularization}

In our final set of simulation studies, we compare the PGD method with convex regularization methods (implemented via \verb+cvx+). In general,  the \verb+cvx+ based regularization algorithm is significantly slower than the PGD method. This is partly due to the infrastructure of generic \verb+cvx+ is not tailored to solve the specific convex optimization problems. On the other hand, the PGD is much easier to implement and enjoys fast rates of convergence, which may also contribute to its improved performance in terms of run-time. Besides, \verb+cvx+ cannot handle $p$ as large as those in Cases 1a, 2a and 3a. Hence, in order to do comparison in terms of the estimation error, we resort to moderate $p$ so that \verb+cvx+ runs to completion. The simulation setup is as follows:

\begin{enumerate}
\item[Case 6a:] (Gaussian) Low CP Rank with  $p=10$, $n =1000$, $r=5$, $\sigma = 0.5, 1, \text{ or } 2$ (SNR $\approx$ 4.8, 2.4 or 1.2);
\item[Case 7a:] (Gaussian) Low Tucker Rank with $p=10$, $n=1000$, $r=5$, $\sigma = 2.5, 5, \text{ or } 10$ (SNR $\approx$ 7.2, 3.6, or 1.8);
\item[Case 8a:] (Gaussian) Slices of Low-rank Matrices with $p=10$, $n=1000$, $r=5$, $s= 5$, $\sigma = 0.5, 1, \text{ or } 2$ (SNR $\approx$ 9.6, 4.8, or 2.4);
\item[Case 6b:] (Logistic) Low CP Rank with  $p=10$, $n =1000$, $\alpha = 3.5$, $r=5$, $m = 20, 5, \text{ or } 1$ (SNR $\approx$ 9.0, 4.5 or 2.0);
\item[Case 7b:] (Logistic) Low Tucker Rank with $p=10$, $n=1000$, $\alpha = 0.5$, $r=5$, $m = 20, 5, \text{ or } 1$ (SNR $\approx$  9.6 , 4.9 or 2.2);
\item[Case 8b:] (Logistic) Slices of Low-rank Matrices with $p=10$, $n=1000$, $\alpha = 1.2$, $r=5$, $s= 5$, $m = 20, 5, \text{ or } 1$ (SNR $\approx$ 7.7 ,  3.8 , or 1.7);
\item[Case 6c:] (Poisson) Low CP Rank with  $p=10$, $n =1000$, $\alpha = 0.5 $, $r=5$, $m = 20, 5, \text{ or } 1$ (SNR $\approx$  9.6 , 4.7, or 2.1);
\item[Case 7c:] (Poisson) Low Tucker Rank with $p=10$, $n=1000$, $\alpha = 0.06$, $r=5$, $m = 20, 5, \text{ or } 1$ (SNR $\approx$ 9.0, 4.5 or 2.0);
\item[Case 8c:] (Poisson) Slices of Low-rank Matrices with $p=10$, $n=1000$, $\alpha = 0.25$, $r=5$, $s= 5$, $m = 30, 10, \text{ or } 5$ (SNR $\approx$ 15.4, 8.8 or  6.2).
\end{enumerate}

Cases 6(a,b,c), 7(a,b,c) and 8(a,b,c) were constructed to represent different types of tensor low-rankness structure in least-square, logistic and Poisson regression. In each case, three levels of SNR(high, moderate and low) are considered. For each setting, we simulated 50 groups of $(T, \epsilon, X)$ and run PGD and convex-regularization methods for the recovery of $T$ to get average rmse with standard deviation for the two approaches respectively. Here we are comparing the best performance achieved by the PGD and convex regularization method respectively: for the PGD we use true parameter as the constraint parameter $r' = r$ (and $s' = s$); for convex regularization method, we do a grid search to choose the tuning parameter that yields the smallest rmse. 

The results are summarized in Table \ref{table1}. They show that in general, the PGD method produces smaller rmse's than convex regularization methods regardless of the noise level of the data.


\begin{table}
\centering
\begin{tabular}{|c|c|c|c|c|}
	\hline
	rmse (sd)  & SNR & PGD & convex-regularization \\
	\hline
	\hline
	Case 6a & High & 0.11 (0.01) & 0.28 (0.02)  \\
	\cline{2-4}
	& Moderate & 0.22 (0.01) &  0.47 (0.02)\\
	\cline{2-4}
	& Low & 0.46 (0.03) &  0.69 (0.02)\\
	\hline
	\hline
	Case 7a & High & 0.07 (0.01) &  0.18 (0.01) \\
	\cline{2-4}
	& Moderate & 0.14 (0.01) &  0.32 (0.02)  \\
	\cline{2-4}
	& Low & 0.28 (0.02) &  0.51 (0.02) \\
	\hline
	\hline
	Case 8a & High & 0.08 (0.01) &  0.12 (0.01)   \\
	\cline{2-4}
	& Moderate & 0.16 (0.01) &  0.23 (0.01)   \\
	\cline{2-4} 
	& Low & 0.30 (0.01) &  0.41 (0.02)   \\
	\hline
	\hline
	Case 6b & High & 0.16 (0.01) &  0.44 (0.02) \\
	\cline{2-4}
	& Moderate & 0.20 (0.01)  & 0.54 (0.02)  \\
	\cline{2-4}
	& Low & 0.35 (0.02) & 0.66 (0.02) \\
	\hline
	\hline
	Case 7b & High & 0.17 (0.01)  &  0.46 (0.02) \\
	\cline{2-4}
	& Moderate & 0.22 (0.01)&   0.55 (0.02)  \\
	\cline{2-4}
	& Low & 0.35 (0.01) & 0.67 (0.01)  \\
	\hline
	\hline
	Case 8b & High & 0.26 (0.01)  & 0.37 (0.02)    \\
	\cline{2-4}
	& Moderate & 0.34 (0.02)  & 0.50 (0.01)    \\
	\cline{2-4} 
	& Low & 0.56 (0.04) &  0.68 (0.02)   \\
	\hline
	\hline
	Case 6c & High & 0.09 (0.01) & 0.57 (0.03)   \\
	\cline{2-4}
	& Moderate & 0.17 (0.01)  &   0.61 (0.04) \\
	\cline{2-4}
	& Low &  0.39 (0.04) &  0.71 (0.03) \\
	\hline
	\hline
	Case 7c & High & 0.12 (0.01)  & 0.74 (0.02)  \\
	\cline{2-4}
	& Moderate &  0.21 (0.02)&  0.75 (0.02)   \\
	\cline{2-4}
	& Low & 0.43 (0.06) &  0.80 (0.02) \\
	\hline
	\hline
	Case 8c & High & 0.13 (0.01)  & 0.79 (0.03)   \\
	\cline{2-4}
	& Moderate &  0.22 (0.03)  &  0.81 (0.03)   \\
	\cline{2-4} 
	& Low & 0.32 (0.03) & 0.83 (0.02)  \\
	\hline
\end{tabular}
\caption{rmse of nonconvex PGD vs convex regularization }
\label{table1}
\end{table}

\section{Proofs}

\label{SecProofs}

\subsection{Proof of general results}
We first prove the results of Section \ref{SecMain}: Theorems \ref{TheoremGeneralLoss}, \ref{TheoremGLM} and \ref{Theorem0}. In particular, we first provide a proof for Theorem \ref{TheoremGeneralLoss}. For convenience we first state the proof for the Gaussian case (Theorem \ref{Theorem0}) and then describe the necessary changes needed for the more general GLM case (Theorem \ref{TheoremGLM}).

\subsubsection{Proof of Theorem \ref{TheoremGeneralLoss}}

The proof follows very similar steps to those developed in ~\cite{JainEtAl14, JainEtAl16}. Recall that $\widehat T_{k+1} = P_{\Theta(t_1 )}(g_k)$ where $g_k = \widehat T_k - \eta \nabla f(\widehat T_k)$. For $\widehat T_{k+1}\in \Theta(t_1 )$ and any $T \in  \Theta(t_0-t_1 )$, the superadditivity condition guarantees that there exists a linear subspace $\mathcal{A} = \{ \alpha_1 \widehat T_{k+1} +\alpha_2 T \ |  \alpha_1, \alpha_2 \in \RR \}$ such that
 $\widehat T_{k+1} \in \mathcal{A}$, $T \in \mathcal{A}$ and $\mathcal{A} \subset \Theta(t_0 )$. 

The contractive projection property CPP($\delta$) implies that for any $T \in  \Theta(t_0-t_1 )$,
$$
\|(\widehat T_{k+1} - g_k )_\mathcal{A} \|_{\rm F} \leq \delta\left \|  \frac{t_0-t_1}{t_1}  \right \|_{\ell_\infty}^{1/2}      \cdot \|(T - g_k )_\mathcal{A} \|_{\rm F}.
$$
Since $t_1=\left\lceil{4 \delta^2 C_u^2 C_l^{-2}\over 1+ 4 \delta^2 C_u^2 C_l^{-2}}\cdot t_0\right\rceil$,
$$
\delta\left \|  \frac{t_0-t_1}{t_1}  \right \|_{\ell_\infty}^{1/2}    \leq (2C_u C_l^{-1})^{-1}.
$$
Hence,
\begin{eqnarray*}
\|\widehat T_{k+1} - T \|_{\rm F} & \leq & \|\widehat T_{k+1} - g_k \|_{\rm F} + \|T - g_k \|_{\rm F}\\
 &\leq& \left(1+  \delta\left \|  \frac{t_0-t_1}{t_1}  \right \|_{\ell_\infty}^{1/2}      \right)\|(T - g_k )_{\mathcal{A}} \|_{\rm F}  \\
 &\leq &  (1+  (2C_u C_l^{-1})^{-1}  )\|(T - g_k )_{\mathcal{A}} \|_{\rm F}  \\
&\leq& \left(1+  \frac{C_l}{2 C_u} \right )\|[T - \widehat T_k -\eta(\nabla f(T) -\nabla f(\widehat T_k)) )_{\mathcal{A}} \|_{\rm F} + 2\eta \| [\nabla f(T)]_{\mathcal{A}} \|_{\rm F},
\end{eqnarray*}
where the final inequality follows from the triangle inequality.
If we define the Hessian matrix of the function $f$ of a vectorized tensor as 
$$
H(A) = \nabla ^2 f(A),
$$
the Mean Value Theorem implies that
$$\mbox{vec}(\nabla f(T) -\nabla f(\widehat T_k)) = H(\widehat T_k +\alpha (T-\widehat T_k))\cdot (T - \widehat T_k),$$
for some $0<\alpha<1$, and
\begin{eqnarray*}
\|\widehat T_{k+1} - T \|_{\rm F}  \leq (1+  (2C_u C_l^{-1})^{-1}  )  \| [(I - \eta H(\widehat T_k +\alpha (T-\widehat T_k)) ) \text{vec}(\widehat T_k-T) ] _{\text{vec}(\mathcal{A})}\|_{\ell_2}\\ +2\eta \| [\nabla f(T)]_{\mathcal{A}} \|_{\rm F}.
\end{eqnarray*}

We now appeal to the following lemma:

\begin{lemma}\label{detail}
Suppose $\calS$ is a linear subspace of $\RR^d$, and $H$ is an $d\times d$ positive semidefinite matrix. For any given $0< c <1$, if for any $x \in \calS$, 
\begin{eqnarray}
\label{eq:cond0}
c  x^\top x \leq  x^\top H x   \leq    (2-c) x^\top x,
\end{eqnarray}
then for any $z \in \calS$, we have
$$
\| [(I - H) z]_{\calS}  \|_{\ell_2} \leq (1- c ) \|z\|_{\ell_2},
$$
$(\cdot)_{\calS}$ stands for the projection onto the subspace $\calS$.
\end{lemma} 

\begin{proof}[Proof of Lemma \ref{detail}]
Suppose the orthonomal basis of $\calS$ is $e_1\ldots, e_q$, and then
$$
\RR^d = \{c e _1| c\in \RR\} \oplus \ldots \oplus \{c e _q| c\in \RR\} \oplus \calS ^\perp
$$
For positive semidefinite $H$, it can be decomposed as follows
$$
H = D^\top D.
$$
Hence we can decompose the rows of $D$ to get
$$
D = \sum_{i=1}^q \lambda_i e_i ^\top + (y_1, \ldots, y_n)^\top
$$
where $y_1,\ldots, y_n \in \calS ^\perp$, and $\lambda_i \in \RR^n$ for $i = 1,\ldots, q$.
Therefore,
$$
D^\top D =  \sum_{i=1}^ q\sum_{j=1}^q  (\lambda_i^\top \lambda_j) e_i e_j^\top + \sum_{k=1}^n y_k y_k^\top  + (y_1, \ldots, y_n) \sum_{i=1}^q \lambda_i e_i^\top + \sum_{i=1}^q e_i \lambda_i^\top \cdot (y_1, \ldots, y_n)^\top.
$$
Now for any $(\alpha_1,\ldots, \alpha_q)^\top \in \RR^q$,  we have $x = \sum_{i=1}^ q \alpha_i e_i\in \calS$, and hence
$$
x^\top D^\top D x = (\sum_{i=1}^ q \alpha_i e_i )^\top D^\top D (\sum_{i=1}^ q \alpha_i e_i ) =  \sum_{i=1}^ q\sum_{j=1}^q  (\lambda_i^\top \lambda_j) \alpha_i \alpha_j.
$$
The equation \ref{eq:cond0} then implies that the matrix 
$$
\Lambda = \{\Lambda_{i,j}\}_{i,j =1}^q \text{ where }
\Lambda_{i,j} =  \lambda_i^\top \lambda_j
$$
has eigenvalues bounded by $c$ from below and $2- c$ from above.
Next, notice that for any $z\in \calS$, we have $z = \sum_{i=1}^q \beta_ i e_i$ for some $(\beta_1,\ldots, \beta_q)^\top \in \RR^q$, and hence due to the fact that $y_1,\ldots y_n \in \calS^\perp$
$$
(I - D^\top D) z =  \left( I -  \sum_{i=1}^ q\sum_{j=1}^q  (\lambda_i^\top \lambda_j) e_i e_j^\top + (y_1, \ldots, y_n) \sum_{i=1}^q \lambda_i e_i^\top \right) \sum_{i=1}^q \beta_ i e_i,
$$
and furthermore
\begin{eqnarray*}
[(I - D^\top D) z]_\calS &=& \left(I - \sum_{i=1}^ q\sum_{j=1}^q  (\lambda_i^\top \lambda_j) e_i e_j^\top\right)    \sum_{i=1}^q \beta_ i e_i\\
&=& (e_1,\ldots, e_q) (I_{q\times q}- \Lambda) (\beta_1,\ldots \beta_q)^\top,
\end{eqnarray*}
and
$$
\|[(I - D^\top D) z]_\calS\|_{\ell_2}^2= (\beta_1,\ldots \beta_q) (I_{q\times q}- \Lambda) (\beta_1,\ldots \beta_q)^\top,
$$
which completes the proof. 
\end{proof}

Condition $RSCS(\Theta(t_0 ), C_l, C_u)$ guarantees the condition of Lemma \ref{detail} is satisfied with $H = \eta H(\widehat T_k +\alpha (T-\widehat T_k)) $, $c = (C_u C_l^{-1})^{-1}$ and $\calS = \mathcal{A}$. Hence 
Lemma \ref{detail} implies that 
\begin{eqnarray*}
\|\widehat T_{k+1} - T\|_{\rm F} &\leq & (1+  (2C_u C_l^{-1})^{-1}  )  \left(1- {(C_u C_l^{-1})^{-1}}\right) \|\widehat T_k -T\|_{\rm F} +2\eta \| [\nabla f(T)]_{\mathcal{A}} \|_{\rm F}.
\end{eqnarray*}

Therefore for any $k$, 
$$
\|\widehat T_{k+1} - T \|_{\rm F} \leq (1- (2C_u C_l^{-1})^{-1}) \| \widehat T_k - T \|_{\rm F}  + 2 \eta Q,
$$
where
\begin{eqnarray*}
 Q &= & \sup_{\mathcal{A}_0\subset \Theta(t_0 )}\| (\nabla f(T))_{\mathcal{A}_0}\|_{\rm F} , 
\end{eqnarray*}
where $\mathcal{A}_0$ is any linear subspace of $\Theta(t_0)$.
We then appeal to the following result.

\begin{lemma}\label{SupLemma}
Suppose $\mathcal A$ is a linear subspace of tensor space $\Omega$. For any $L\in \Omega$,
$$
\| (L)_{\mathcal A} \|_{\rm F}  =  \sup_{A\in \mathcal{A}\cap \mathbb{B}_{\rm F}(1)}\langle A,L\rangle
$$
\end{lemma} 

\begin{proof}[Proof of Lemma \ref{SupLemma}]
First, we are going to show that
$$
\| (L)_{\mathcal A} \|_{\rm F}  \leq  \sup_{A\in {\mathcal A}\cap \mathbb{B}_{\rm F}(1)}\langle A,L\rangle
$$
Suppose we have $(L)_{\mathcal A} = P\in \mathcal A$. Since for any $\alpha >-1$, $P+ \alpha P\in \mathcal A$, and hence
$$
\| P+ \alpha P - L \|_{\rm F} = \|P-L\|_{\rm F} +\alpha^2\|P\|_{\rm F} +\alpha \langle P,P-L \rangle \leq \|P-L\|_{\rm F}
$$
we must have $\langle P,P-L \rangle =0$, i.e. $\langle P, L \rangle  =\langle P,P \rangle$. (otherwise $\alpha$ of small magnitude with the same sign of $\langle P,P-L \rangle$ will violate the inequality).
Therefore, 
$$
\sup_{A\in  \mathcal A\cap \mathbb{B}_{\rm F}(1)}\langle A,L\rangle \geq 
\left\langle \frac{P}{\|P\|_{\rm F}},L \right\rangle = \| P\|_{\rm F} = \|(L)_{\mathcal A}\|_{\rm F}
$$
What remains is to show 
$$
\| (L)_{\mathcal A} \|_{\rm F}  \geq  \sup_{A\in {\mathcal A}\cap \mathbb{B}_{\rm F}(1)}\langle A,L\rangle
$$
For any $D\in {\mathcal A}\cap \mathbb{B}_{\rm F}(1)$, let $D_\alpha$ be the projection of $L$ onto $\{\alpha D | \alpha\geq 0 \}$, and hence
$$
\langle D_\alpha, L \rangle = \langle D_\alpha, D_\alpha \rangle \leq \langle P, P \rangle.
$$
Therefore, we have
$$
\left\langle {D},L \right\rangle \leq \left\langle \frac{D}{\|D\|_{\rm F}},L \right\rangle = \left\langle \frac{D_\alpha}{\|D_\alpha\|_{\rm F}},L \right\rangle \leq \| P\|_{\rm F}
$$
which completes the proof.
\end{proof}

Lemma \ref{SupLemma} then implies
$$
Q= \sup_{\mathcal{A}_0\subset \Theta(t_0 )}\| (\nabla f(T))_{\mathcal{A}_0}\|_{\rm F} =  \sup_{\mathcal{A}_0\subset \Theta(t_0 )} \sup_{A \in \mathcal{A}_0 \cap \mathbb{B}_{\rm F}(1)}  \langle \nabla f(T), A \rangle\le \sup_{A\in \Theta(t_0 )\cap \mathbb{B}_{\rm F}(1)}\left\langle \nabla f(T), A\right\rangle.
$$
Therefore, after
$$K = \lceil 2 C_u C_l^{-1} \log\frac{ {\|T\|}_{\rm F}}{\epsilon}\rceil$$
iterations,
$$
\|\widehat T_k - T \|_{\rm F} \leq \epsilon + 4 \eta C_u C_l^{-1}  \sup_{A\in \Theta(t_0 )\cap \mathbb{B}_{\rm F}(1)}\left\langle \nabla f(T), A\right\rangle,
$$
which completes the proof.
\medskip

\subsubsection{Proof of Theorem \ref{Theorem0}}
Recall that the original least-squares objective is
$$
f(A) = \frac{1}{2n} \sum_{i=1}^n{(Y^{(i)} - \langle X^{(i)}, A \rangle)^2}.
$$
Hence, the gradient at true tensor coefficient $T$ is :
$$
\nabla f(T) = \frac{1}{n} \sum_{i=1}^n {X^{(i)} \otimes [\langle T,  X^{(i)}\rangle - Y^{(i)}]} = -\frac{1}{n} \sum_{i=1}^n {X^{(i)} \zeta^{(i)}}
$$
for the least-squares objective we consider. Further
$$
\nabla^2 f(T) = \frac{1}{n} \sum_{i=1}^{n}{X^{(i)} \otimes X^{(i)}}.
$$
Through vectorization, the Hessian matrix $H$ is
$$
H = \sum_{i=1}^n \text{vec}(X^{(i)}) \text{vec}(X^{(i)})^\top.
$$
Lemma \ref{RestrictedEigenvalueLemma} then implies that given $n \geq c_1 w_G^2[\Theta( t_0) \cap \mathbb{B}_{\rm F}(1)],$ with probability at least
$$1-c_2/2\exp(-c_3w_G^2[\Theta(t_0) \cap \mathbb{B}_{\rm F}(1)]),$$
we have, for any $A \in \Theta(t_0)$,
$$
\left(\tau^{-1} c_l\right)^2 \langle A, A\rangle \leq   \frac{1}{n}  \sum_{i=1}^n \langle X^{(i)},A \rangle ^2      \leq (\tau c_u)^2 \langle A, A \rangle,
$$
which guarantees the $RSCS(\Theta(t_0 ), C_l, C_u)$ condition with $C_u = \tau c_u$ and $C_l = \tau^{-1} c_l$. 
Thus Theorem \ref{TheoremGeneralLoss} implies that 
$$
\|\widehat T_k - T \|_{\rm F} \leq \epsilon + 4 \eta \tau^2\kappa  \sup_{A\in \Theta(t_0 )\cap \mathbb{B}_{\rm F}(1)}\left\langle \frac{1}{n}\sum_{i=1}^n \zeta^{(i)} X^{(i)}, A\right\rangle.
$$
The last step is to show that
$$
\sup_{A\in \Theta(t_0 )\cap \mathbb{B}_{\rm F}(1)}\left\langle \frac{1}{n}\sum_{i=1}^n \zeta^{(i)}X^{(i)}, A\right\rangle \leq 2 c_u \sigma n^{-1/2} w_G[\Theta(t_0) \cap \mathbb{B}_{\rm F}(1)],
$$
with probability at least 
$$1 -c_2/2\exp\left\{-c_3 {w_G^2[\Theta(t_0) \cap \mathbb{B}_{\rm F}(1)] }\right\}.$$ 
This can be shown by simply applying Lemma 11 in ~\cite{RaskuttiYuan15} and replacing $\{A| \mathcal{R}(A)\leq 1\}$ with $\Theta(t_0 )\cap \mathbb{B}_{\rm F}(1)$. Note that all the proof steps for Lemma 11 in ~\cite{RaskuttiYuan15} are identical for $\Theta(t_0 )\cap \mathbb{B}_{\rm F}(1)$ since the sets $\Theta(t)$'s are symmetric.

\subsubsection{Proof of Theorem \ref{TheoremGLM}}

The proof follows the same flow as Theorem \ref{Theorem0} but we requires an important concentration result from \cite{Mendelson15}. Recall that in the GLM setting, according to \eqref{Eq:GLMLogLike},
$$
f(A) = \frac{1}{n} \sum_{i=1}^n  {(a(\langle X^{(i)}, A \rangle) - Y^{(i)}\langle X^{(i)}, A \rangle)}.
$$
Hence the gradient at true coefficient $T$ is 
$$
\nabla f(T) = \frac{1}{n} \sum_{i=1}^n  {(\mu_i - Y^{(i)})}X^{(i)},
$$
where $\mu_i = a'(\langle X^{(i)}, T \rangle)$, and the Hessian matrix at vectorized tensor $T$ is
$$
\nabla^2 f(T) =\sum_{i=1}^n W_{ii} \text{vec}(X^{(i)}) \text{vec}(X^{(i)})^\top.
$$
where $W_{ii}= a''(\langle X^{(i)}, T \rangle)$. 

Since $\mbox{Var}(Y^{(i)}) = a''(\langle X^{(i)}, T \rangle) = W_{ii}$, the moment assumption
${1}/{\tau_0^2} \leq {\rm Var}(Y^{(i)}) \leq \tau_0^2$ guarantees that 
$$
\frac{1}{\tau_0^2} \leq W_{ii} \leq \tau_0^2.
$$
Plus, for any $\tau > \tau_0$, Lemma~\ref{RestrictedEigenvalueLemma} guarantees that when $n > c_1w_G^2[\Theta(t_0) \cap \mathbb{B}_{\rm F}(1)]$,
$$
\left((\tau/\tau_0)^{-1} c_l\right)^2 \langle A, A\rangle \leq   \frac{1}{n}  \sum_{i=1}^n \langle X^{(i)},A \rangle ^2      \leq ((\tau/ \tau_0) c_u)^2 \langle A, A \rangle.
$$
Therefore $RSCS(\Theta(t_0 ), C_l, C_u)$ condition holds with 
$C_l = \tau^{-1} c_l$ and $C_u = \tau c_u$. Thus Theorem \ref{TheoremGeneralLoss} implies that 
$$
\|\widehat T_k - T \|_{\rm F} \leq \epsilon + 4 \eta \tau^2\kappa  \sup_{A\in \Theta(t_0 )\cap \mathbb{B}_{\rm F}(1)}\left\langle \frac{1}{n} \sum_{i=1}^n  {(Y^{(i)}-\mu_i)}X^{(i)}, A\right\rangle.
$$
For the last step, by applying a concentration result on the following multiplier empirical process
$$
 \sup_{A\in \Theta(t_0 )\cap \mathbb{B}_{\rm F}(1)}\left\langle \frac{1}{n} \sum_{i=1}^n  {(Y^{(i)}-\mu_i)}X^{(i)}, A\right\rangle =  \sup_{A\in \Theta(t_0)\cap \mathbb{B}_{\rm F}(1)}  \frac{1}{n} \sum_{i=1}^n{(Y^{(i)}-\mu_i)}  \left\langle    X^{(i)}, A\right\rangle,
$$
we can bound the quantity by the Gaussian width with large probability, up to some constant. 

More specifically, denote 
$$
\omega^{(i)} = \Sigma^{-1/2}\text{vec}(X^{(i)}) 
$$
then $\{\omega^{(i)}\}_{i=1}^n$ are i.i.d Gaussian random vectors and hence
\begin{eqnarray*}
& &\sup_{A\in \Theta(t_0 )\cap \mathbb{B}_{\rm F}(1)}  \frac{1}{n} \sum_{i=1}^n{(Y^{(i)}-\mu_i)}  \left\langle    X^{(i)}, A\right\rangle \\
&=& \sup_{F\in \text{vec}(\Theta(t_0 )\cap \mathbb{B}_{\rm F}(1))}  \frac{1}{n} \sum_{i=1}^n{(Y^{(i)}-\mu_i)}     \text{vec}(X^{(i)})^\top F\\
&=& \sup_{F\in \Sigma^{1/2} \cdot \text{vec}(\Theta(t_0 )\cap \mathbb{B}_{\rm F}(1))}  \frac{1}{n} \sum_{i=1}^n{(Y^{(i)}-\mu_i)}     (\omega^{(i)})^\top F\\
&\leq& c_5 M_Y^{1/q} \sup_{F\in \Sigma^{1/2} \cdot \text{vec}(\Theta(t_0 )\cap \mathbb{B}_{\rm F}(1))}  \frac{1}{n} \sum_{i=1}^n    (\omega^{(i)})^\top F\\
&= & c_5 M_Y^{1/q} \sup_{A\in \Theta(t_0 )\cap \mathbb{B}_{\rm F}(1)}  \frac{1}{n} \sum_{i=1}^n  \left\langle    X^{(i)}, A\right\rangle\\
&\leq& \frac{c_5 M_Y^{1/q}c_u}{\sqrt{n}} w_G^2[\Theta( t_0) \cap \mathbb{B}_{\rm F}(1)],
\end{eqnarray*}
where the first inequality follows from Theorem 1.9 of \cite{Mendelson15}, and the second inequality holds in view of Lemma 11 of \cite{RaskuttiYuan15}, and both inequalities hold with probability greater than
$$1- c_2\exp\left\{-c_3 {w_G^2[\Theta(  \theta' +  \theta) \cap \mathbb{B}_{\rm F}(1)] }\right\} - c_4 n^{-(q/2-1)}\log^q n.$$

\subsection{Proofs of results in Section \ref{SecExamples}} 
We now present the proofs for the two main examples $\Theta_2(r ,s)$ and $\Theta_3(r)$. Our proofs involve: (i) proving that the projection properties hold for both sets of cones and (ii) finding an upper bound for the Gaussian width $w_G[\Theta(t) \cap \mathbb{B}_{\rm F}(1)]$. 

\subsubsection{Proof of Theorem \ref{Theorem1}}

First, it is straightforward to verify that $\{\Theta_2(r,s)\}$ is a superadditive family of symmetric cones. We then verify the contraction properties as stated by Lemma \ref{ProjectionLemma1}.

\begin{proof}[Proof of Lemma \ref{ProjectionLemma1}]
We need to develop an upper bound for $\|P_{ \Theta_2(r_2,s_2)}(Z) - Z \|_{\rm F}$ for a general tensor $Z \in \Theta_2(r_0,s_0)$. Let $\tilde{Z} \in \RR^{d_1\times d_2\times d_3}$ denote the tensor whose slices $\tilde{Z}_{\cdot \cdot j_3}$, ($j_3 = 1,\ldots, d_3$) are the rank-$r_2$ approximation of the corresponding slices of $Z$. First, it follows from the contraction property of low rank matrix projector \citep[see, e.g.,][]{JainEtAl16} that for all $1 \leq j_3 \leq d_3$, for any $Y_{\cdot \cdot j_3}$ such that $\mbox{rank}(Y_{\cdot \cdot j_3}) \leq r_1$
$$
\|\tilde{Z}_{\cdot \cdot j_3} - Z_{\cdot \cdot j_3}\|_{\rm F} \leq \beta \|Y_{\cdot \cdot j_3} - Z_{\cdot \cdot j_3}\|_{\rm F}.
$$
By summing over $j_3$ it follows that for any $Y \in \RR^{d_1\times d_2\times d_3}$ where $\mbox{rank}(Y_{\cdot \cdot j_3}) \leq r_1$ for all $j_3$
$$
\|\tilde{Z} - Z\|_{\rm F} \leq \beta \|Y - Z\|_{\rm F}.
$$

The projection $P_{ \Theta_2(r_2,s_2)}(Z)$ involves zeroing out the slices of $\tilde{Z}$ with the smallest magnitude. Let $v_{\tilde{Z}} := \mbox{vec}(\|\tilde{Z}_{\cdot \cdot 1} \|_{\rm F},\|\tilde{Z}_{\cdot \cdot 2} \|_{\rm F}, \ldots, \|\tilde{Z}_{\cdot \cdot d_3} \|_{\rm F} )$. As shown by \cite{JainEtAl14}, for all $Y$ where $v_Y = \mbox{vec}(\|Y_{\cdot \cdot 1} \|_{\rm F},\|Y_{\cdot \cdot 2} \|_{\rm F}, \ldots, \|Y_{\cdot \cdot d_3} \|_{\rm F} )$ and $\|v_Y\|_{\ell_0} \leq s_1$,
$$
\|\tilde{P}_{s}(v_{\tilde{Z}}) - v_{\tilde{Z}}\|_{\ell_2} \leq \alpha\| v_Y- v_{\tilde{Z}}\|_{\ell_2}.
$$
Therefore
\begin{eqnarray*}
\| P_{ \Theta_2(r_2,s_2)}(Z) - \tilde{Z} \|_{\rm F} &\leq& \alpha \| Y - \tilde{Z}   \|_{\rm F} \\
&\leq& \alpha (\|Y-Z\|_{\rm F} +\|\tilde{Z}-Z\|_{\rm F})\\
&\leq& (\alpha +\alpha\beta) \|Y-Z\|_{\rm F}.
\end{eqnarray*}
Hence using the triangle inequality:
$$
\| P_{ \Theta_2(r_2,s_2)}(Z) - Z \|_{\rm F} \leq  \|\tilde{Z} - Z\|_{\rm F} +\| P_{ \Theta_2(r_2,s_2)}(Z) - \tilde{Z} \|_{\rm F} \leq (\alpha +\beta +\alpha\beta)\cdot \|Y - Z\|_{\rm F},
$$
which completes the proof.
\end{proof}

Lemma \ref{ProjectionLemma1} guarantees that $P_{\Theta_2(r,s)}$ satisfies the contractive projection property CPP($\delta$) with $\delta =3$. Hence, by setting $t_1 = (r',s' )$ and $t_0 =  ( r' + r , s' + s)$, Theorem \ref{Theorem0} directly implies the linear convergence rate result with statistical error of order
$$n^{-1/2}{w_G[\Theta_2(  r' + r , s' + s) \cap \mathbb{B}_{\rm F}(1)] }.$$

It remains to calibrate the Gaussian width. Recall the definition of the convex regularizer:
$$
\mathcal{R}_1(A) = \sum_{j_3 = 1}^{d_3} \| A_{\cdot \cdot j_3} \|_{\ast}.
$$
It is straightforward to show that
$$
\Theta_2(r' + r,s'+ s)\cap \mathbb{B}_{\rm F}(1) \subset \mathbb{B}_{\mathcal{R}_1}(\sqrt{(r' + r)(s'+ s)}).
$$
Then Lemma 5 of \cite{RaskuttiYuan15} implies that 
\begin{eqnarray*}
w_G[\Theta_2(  r' + r , s' + s) \cap \mathbb{B}_{\rm F}(1)] &\leq&  w_G[\mathbb{B}_{\mathcal{R}_1}(\sqrt{(r' + r)(s'+ s)})]  \\ &\leq&  \sqrt{(s' + s)(r'+r)}w_G[\mathbb{B}_{\mathcal{R}_1}(1)] \\ &\leq&  \sqrt{(s' + s)(r'+r)} \sqrt{{6(d_1 +d_2 +\log d_3)}}
\end{eqnarray*}
which completes the proof.

\subsubsection{Proof of Theorem \ref{Theorem2}}

Once again, it is straightforward to verify that $\{\Theta_3(r)\}$ is a superadditive family of symmetric cones. We now verify the contraction properties

\begin{proof}[Proof of Lemma \ref{ProjectionLemma2}]
To develop an upper bound for $\|\widehat{P}_{ \Theta_3(r_2)}(Z) - Z \|_{\rm F}$ for a general tensor $Z \in \Theta_2(r_0)$, we introduce the following three tensors (recursively):
\begin{eqnarray*}
Z_{(1)} &:=&  (\mathcal{M}_1^{-1}\circ\bar{P}_{r_2}\circ\mathcal{M}_1)(Z)\\
Z_{(2)} &:=& (\mathcal{M}_2^{-1}\circ\bar{P}_{r_2}\circ\mathcal{M}_2)(Z_{(1)})\\
Z_{(3)} &:=& (\mathcal{M}_3^{-1}\circ\bar{P}_{r_2}\circ\mathcal{M}_3)(Z_{(2)}),        
\end{eqnarray*}
where we recall that $\mathcal{M}_1(\cdot)$, $\mathcal{M}_2(\cdot)$ and $\mathcal{M}_3(\cdot)$ are the mode-1, mode-2 and mode-3 matricization operators. Therefore $\widehat{P}_{ \Theta_3(r_2)}(Z) = Z_{(3)}$ and:
\begin{eqnarray*}
\|\widehat{P}_{ \Theta_3(r_2)}(Z) - Z \|_{\rm F} \leq \|\widehat{P}_{ \Theta_3(r_2)}(Z) - Z_{(2)}\|_{\rm F} + \|Z_{(2)} - Z_{(1)}\|_{\rm F} + \|Z_{(1)}-Z\|_{\rm F}.
\end{eqnarray*}
Next note that
\begin{eqnarray*}
\|Z_{(1)}-Z\|_{\rm F} = \| (\mathcal{M}_1^{-1}\circ\bar{P}_{r_2}\circ\mathcal{M}_1)(Z) - Z \|_{\rm F} = \|\bar{P}_{r_2}(\mathcal{M}_1(Z)) - \mathcal{M}_1(Z)\|_{\rm F}.
\end{eqnarray*}
As shown by \cite{JainEtAl16}, for any $Y$ such that $\mbox{rank}(Y) \leq r_1$,
\begin{eqnarray*}
\|Z_{(1)}-Z\|_{\rm F} \leq \beta \|Y - Z\|_{\rm F}.
\end{eqnarray*}
Using a similar argument and the triangle inequality
\begin{eqnarray*}
\|Z_{(2)}-Z_{(1)}\|_{\rm F} \leq \beta \|Y - Z_{(1)}\|_{\rm F} \leq \beta(\|Y - Z\|_{\rm F} +  \|Z_{(1)} - Z\|_{\rm F}) \leq (\beta + \beta^2)\|Y-Z\|_{\rm F}.
\end{eqnarray*}
Furthermore,
\begin{eqnarray*}
\|Z_{(3)}-Z_{(2)}\|_{\rm F} &\leq& \beta \|Y - Z_{(2)}\|_{\rm F}\\
&\leq& \beta(\|Y - Z\|_{\rm F} +  \|Z_{(2)} - Z\|_{\rm F})\\
&\leq& (\beta + 2\beta^2 + \beta^3)\|Y-Z\|_{\rm F}.
\end{eqnarray*}
Therefore for all $Y \in {\Theta_3(r_1)}$
\begin{eqnarray*}
\|\widehat{P}_{ \Theta_3(r_2)}(Z) - Z \|_{\rm F} &=& \|P_{ \Theta_3(r_2)}(Z) - Z_{(2)}\|_{\rm F} + \|Z_{(2)} - Z_{(1)}\|_{\rm F} + \|Z_{(1)}-Z\|_{\rm F}\\
&\leq& (3 \beta + 3 \beta^2 + \beta^3)\|Y - Z\|_{\rm F}.
\end{eqnarray*}
\end{proof}

Lemma \ref{ProjectionLemma1} guarantees the approximate projection $\widehat P_{\Theta_3(r)}$ fulfills the contractive projection property CPP($\delta$) with $\delta =7$. And hence via setting $t_1 = r'$ and $t_0 = r' + r$, Theorem \ref{Theorem0} directly implies the linear convergence rate result with statistical error of order $ n^{-1/2}{w_G[\Theta_3(  r' + r ) \cap \mathbb{B}_{\rm F}(1)] }$.
To upper bound the Gaussian width, we define the following nuclear norms:
$$
\mathcal{R}_{(i)}(A) = \|\mathcal{M}_i(A)\|_{*},
$$
where $1 \leq i \leq 3$ and $\|.\|_{*}$ is the standard nuclear norm. Then it clearly follows that
$$
\Theta_3(r'+ r)\cap \mathbb{B}_{\rm F}(1)\subset  \cap_{i=1}^{3} \mathbb{B}_{\mathcal{R}_{(i)}}(\sqrt{r' + r}).
$$
Lemma 5 in ~\cite{RaskuttiYuan15} then implies that 
\begin{eqnarray*}
w_G[\Theta_3(  r' + r ) \cap \mathbb{B}_{\rm F}(1)] &\leq&  w_G[\cap_i \mathbb{B}_{\mathcal {R}_{(i)}}(\sqrt{r' + r})]  \\ 
&\leq & \min_i w_G[ \mathbb{B}_{\mathcal {R}_{(i)}}(\sqrt{r' + r})] \\
&\leq&  \sqrt{r'+r}\min_i w_G[ \mathbb{B}_{\mathcal {R}_{(i)}}(1)] \\ 
&\leq&  \sqrt{r'+r} \sqrt{6\min\{ d_1+d_2d_3, d_2 + d_1d_3, d_3+ d_1d_2\}}
\end{eqnarray*}
which completes the proof.

\bibliographystyle{plainnat}

\bibliography{Biblio_TensorPGD}
\end{document}